\documentclass[12pt]{article}
\usepackage[T1]{fontenc}

\RequirePackage{amsmath,amsthm,amsfonts,amssymb,bm,mathrsfs}

\usepackage{dsfont}
\usepackage{bm}
\usepackage{color,soul}
\usepackage{subfig}
\usepackage{hyperref}
\usepackage{multirow}
\usepackage{natbib}
\usepackage{enumitem}
\newcommand{\blind}{1}

\usepackage{crimson}
\usepackage[T1]{fontenc}

\newtheorem{thm}{Theorem}
\newtheorem{cor}{Corollary}

\newtheorem{prop}{Proposition}

\theoremstyle{lemma}
\newtheorem{Lemma}{Lemma}

\theoremstyle{remark}
\newtheorem{remark}{\bf Remark}

\theoremstyle{example}

\theoremstyle{claim}
\newtheorem{claim}{Claim}

\theoremstyle{assumption}
\newtheorem{Assumption}{Assumption}
\newtheorem{condition}{Condition}

\newtheorem{Definition}{Definition}
\usepackage{enumerate}

\usepackage{mathtools}
\DeclareMathOperator*{\argmax}{arg\,max\;}

\def\R{{\mathbb R}}

\def\Z{{\mathbb Z}}
\def\N{{\mathbb N}}

\def\E{{\mathbb E}}
\def\R{{\mathbb R}}

\def\P{{\mathbb P}}

\def\b{{\mathbf b}}
\def\a{{\mathbf a}}

\def\v{{\bm{v}}}

\def\EE{{\mathcal{E}}}

\def\u{{\bm{u}}}
\def\w{{\bm{w}}}
\def\deg{{\text{deg}}}
\def\a{{\alpha}}
\def\b{{\beta}}

\DeclarePairedDelimiter\abs{\lvert}{\rvert}%
\DeclarePairedDelimiter\norm{\lVert}{\rVert}%
\makeatletter
\let\oldabs\abs
\def\abs{\@ifstar{\oldabs}{\oldabs*}}
\let\oldnorm\norm
\def\norm{\@ifstar{\oldnorm}{\oldnorm*}}
\makeatother

\newcommand{\Ber}{\text{Ber}}

\newcommand{\bu}{\boldsymbol{u}}

\newcommand{\vol}{\text{vol}}

\definecolor{brickred}{rgb}{0.50, 0.22, 0.10}
\usepackage[titletoc]{appendix}

\usepackage[dvipsnames]{xcolor}

\begin{document}

\def\spacingset#1{\renewcommand{\baselinestretch}%
	{#1}\small\normalsize} \spacingset{1}


\if1\blind
{
\title{\bf  A General Pairwise Comparison Model for Extremely Sparse Networks}
\author{Ruijian $\rm{Han}^{a\dag}$\thanks{Corresponding author. Email: ruijianhan@cuhk.edu.hk. Address: Department of Statistics, The Chinese University of Hong Kong, Shatin, N.T., Hong Kong, China},\hspace{0.2cm}
	Yiming $\rm{Xu}^b$\thanks{Equal Contribution.},\hspace{0.2cm}
	and
	Kani $\rm{Chen}^c$\\
	\\
{\small {\small {$\it^{a}$ Department of Statistics, The Chinese University of Hong Kong, Hong Kong, China} }}\\
{\small {\small {$\it^{b}$  Department of Mathematics, University of Utah, Salt Lake City, USA } }}\\
{\small {\small {$\it^{c}$  Department of Mathematics,  } }}\\
{\small {\small  Hong Kong University of Science and Technology, Hong Kong, China }} }
}

\date{}
\maketitle
\fi

\if0\blind
{
	\bigskip
	\bigskip
	\bigskip
	\begin{center}
		{\LARGE\bf  A General Pairwise Comparison Model for Extremely Sparse Networks}
	\end{center}
	\medskip
} \fi

\bigskip
\begin{abstract}
Statistical inference using pairwise comparison data is an effective approach to analyzing large-scale sparse networks. 
In this paper, we propose a general framework to model the mutual interactions in a network, which enjoys ample flexibility in terms of model parametrization. Under this setup, we show that the maximum likelihood estimator for the latent score vector of the subjects is uniformly consistent under a near-minimal condition on network sparsity. 
This condition is sharp in terms of the leading order asymptotics describing the sparsity. Our analysis utilizes a novel chaining technique and illustrates an important connection between graph topology and model consistency. Our results guarantee that the maximum likelihood estimator is justified for estimation in large-scale pairwise comparison networks where data are asymptotically deficient. Simulation studies are provided in support of our theoretical findings. 
\end{abstract}

\noindent%


\vfill

\newpage
\spacingset{1.1} 

\section{Introduction}

Pairwise comparison data arise frequently in network data analysis and assist people in finding vital information underlying many modern interaction systems such as social webs and sports tournaments. 
For example, match points can be used to evaluate the team's strengths in sports competitions. Specifically,
consider $[n]=\{1, \cdots, n\}$ as $n$ subjects in a network of interest, and assume that every $i \in [n]$ is assigned a latent score $u_i$, and $\u=(u_i)^T_{i\in [n]}$ denotes the corresponding score vector. 
Given pairwise comparison data, defined as a set of independent random variables $\{X_{ij}\}_{1\leq i<j\leq n}$, one wishes to accrue knowledge on $\u$ via statistical inference processes. In the case of team ranking,  $\u$ is a vector measuring the strength of $n$ teams and $X_{ij}$ is the competition outcome between teams $i$ and $j$, which is a binary random variable (denoting win and loss) depending on $u_i$ and $u_j$. Whereas in other scenarios such as online assessment, $X_{ij}$ represents the average rating of subject $i$ against subject $j$. Under such circumstances, a continuous spectrum of the rating outcome is more appropriate. 

Mutual comparison data from different sources may take different forms and enhance our understanding of the model as long as it is relevant to $\u$, and this is usually manifested in assumptions on the link function between $\u$ and $\{X_{ij}\}_{1\leq i<j\leq n}$. 
Statistical methods can then be applied to estimate $\u$. 
The framework described here provides a simplified characterization of many parametric pairwise comparison models in the literature.

The study of parametric pairwise comparison models emerged in the early 20th century and gradually gained popularity. 
Among them, the Bradley-Terry (BT) model, first introduced in \cite{MR0070925}, attracted much attention.
The BT model is a specification of the team ranking example introduced before:
\begin{equation*}
\P\left(\text{team $i$ beats team $j$}\right)=\P\left(X_{ij}=1\right)= \Phi(u_i-u_j),
\end{equation*} 
where $\Phi(\cdot)$ is the logistic link function defined as $\Phi(x)=(1+e^{-x})^{-1}$. In practice, not every pair of subjects admits a comparison, and
even if it does,  not all comparison data can be observed. 
To accommodate such sampling scenarios, the comparison graph is often assumed as an incomplete general graph;
a comparison outcome between two subjects is observed if there is an edge between them. 
A common approach to modeling the comparison graph structure is through a generalized random graph.
In particular, each pair $(i, j)$ is associated with an independent Bernoulli random variable $n_{ij}\sim \Ber(p_{ij,n})$
to determine the availability of $X_{ij}$: $n_{ij}=1$ if $X_{ij}$ is observed and $n_{ij}=0$ otherwise. 
Comparison rates $p_{ij,n}$ are constants and measure the (edge) density of a network. 
{ When $p_{ij,n}$ is independent of $i$ and $j$, namely $ p_{ij,n} = p_n$, the comparison graph is homogeneous and is called the Erd\H{o}s-R\'enyi graph $G(n, p_n)$.}
A detailed discussion of the BT models with random graph structure can be found in \cite{MR947340}. 

Despite its simplicity, the BT model fails to capture the truth when $X_{ij}$ is not binary.  
The specific choice of $\Phi(x)$ also limits the use of the BT model in practice. 
To generalize, numerous variants of the BT model have been invented either by considering Likert-scale response or by replacing $\Phi(x)$ with a different link function \citep{thurstone1927law,MR217963,davidson1970extending,MR1064798}.   
These choices are usually case-dependent and a unified treatment is yet to be found, which gives the motivation for the current paper. Before going further to explain how to generalize the BT model as well as developing a consistent estimation theory under the generalized framework, we recall a few existing inference results in the BT model.


A natural estimator for $\u$ is the maximum likelihood estimator (MLE), which is denoted by $\widehat{{\u}}$ and will be the main focus in this paper. 
{The consistency of the MLE in the BT model has been well studied when the comparison graph is the Erd\H{o}s-R\'enyi graph $G(n, p_n)$}. For instance, it was established in \cite{MR1724040,MR2987494} that the MLE is consistent in the $\ell_\infty$ norm, also called the uniformly consistent, if $\liminf_{n\to \infty}p_n>0$. Consequently, for a network of $n$ subjects, at least $cn^2$ ($c>0$) samples are needed to ensure the convergence of the MLE based on their results. The quadratic requirement on samples can be restrictive even when $n$ is only moderately large. 
In fact, many large-scale networks arising from realistic applications are sparse, that is, the degree of most subjects is sublinear in the size of the network.

To address the issue on sparsity, some researchers studied consistency of the MLE under a weaker condition on $p_n$ but in a different metric; see \cite{maystre2015fast, MR3613103}. They showed that $\|\widehat{\u}-\u\|_2$ cannot grow faster than $\sqrt{n(\log n)^{-\kappa}}$ if $p_n\geq n^{-1}(\log n)^{\kappa +1}$ ($\kappa>0$). 
{A similar result for the $\ell_2$ norm weighted by the graph Laplacian is obtained in \cite{MR3504618} without requiring homogeneity on $p_{ij,n}$.} Although the sparsity condition in these results is optimal, the (weighted) $\ell_2$ norm only reflects the averaging behavior of the estimator, from which one cannot deduce convergence for each component unless the scores of all subjects are of the same order. 

Recently, \cite{BTL} and \cite{MR3953449} made progress by establishing the uniform consistency of the MLE and the regularized MLE in sparse homogenous BT models, respectively. 
Their analysis heavily depends on the special parametrization of the BT model as well as the homogeneity assumption on the comparison rates. 
It is also worth noting that the uniform consistency of the regularized MLE mentioned above \emph{does not} directly imply the same result for the MLE.
\cite{chen2020partial} improved the sparsity condition in \cite{BTL} and showed that the MLE is superior to the spectral method in \cite{MR3613103} in terms of the multiplicative constant factors in the sample complexity.


In this paper, we develop a uniform consistency theory of the MLE for a general class of comparison models under a near-optimal sparsity condition. Our contribution can be briefly summarized as follows: 

\begin{itemize}
	\item We build a general probabilistic framework for pairwise comparison network data analysis. 
	Our framework enjoys sufficient flexibility in terms of model parametrization, covering a wide variety of existing models in the literature such as the BT model, the Thurstone-Mosteller model \citep{thurstone1927law,mosteller2006remarks}, the Davidson model \citep{davidson1970extending} and many others.
	\item Under this framework, we identify a sufficient condition for the uniform consistency of the MLE for the latent score vector. 
	Our condition can be succinctly characterized using the graph topology and is satisfied in many random graph ensembles with varying parameters. 
	{In particular, when the sampling model is the Erd\H{o}s-R\'enyi graph $ G(n, p_n) $, we show that the comparison rate $p_n$ can be chosen as small as of order $ (\log n)^{3+\epsilon}/n$ ($\epsilon>0$) to ensure the uniform consistency of the MLE (with convergence rate at least of order $(\log n)^{-\epsilon/2}$), matching the graph connectivity threshold $\log n/n$ up to logarithmic factors.}

\end{itemize}

The uniform consistency results guarantee the entry-wise convergence of the estimator at a uniform rate. 
Our approach decouples the randomness in pairwise comparison modeling and graph sampling, demonstrating a deep connection between graph topology and consistency of the MLE. 
On the practical side, our results justify that the MLE can be used in large-scale complex network estimation even if comparison data is asymptotically deficient,
that is, the ratio of the observed comparisons and the theoretical total comparisons goes to zero (at a certain rate) as the number of subjects goes to infinity.  
In addition, numerical results are provided in strong support of our theory.     

\color{black}
The rest of this paper is organized as follows. 
Section \ref{section:setup} introduces a general framework for parametric pairwise comparison models and an inhomogeneous assumption on the comparison graph. 
Section \ref{section:main} establishes both the unique existence and uniform consistency of the MLE in the setup introduced in Section \ref{section:setup} under a near-optimal sparsity condition. 
Section \ref{section:examples} shows that uniform consistency of the MLE in many existing models can be deduced as corollaries from our results in Section \ref{section:main}.
Section \ref{Comparison graph} discusses a few graph topological conditions that will be used to formulate the uniform consistency result. 
Sections \ref{section:num} and \ref{section:discussion} are devoted to providing numerical experiments to support our theoretical findings and discussing future research directions, respectively. 
The proofs of the main results can be found in Appendices.  

 \color{black}
\subsection{Notation} 


For $m\in\N$, we write $\{1, \cdots, m\}$ as $[m]$.
For sequences $\{a_n\}_{n\in\N}$ and $\{b_n\}_{n\in\N}$, $a_n\lesssim b_n$ if there exists an absolute constant $C$ such that $a_n\leq Cb_n$ for all $n\in\N$. 
Particularly, $a_n=\Omega(b_n)$ if $a_n\lesssim b_n$ and  $b_n\lesssim a_n$. For a univariate function $ f(x) $, $ f'(x) $ and $ f''(x) $ denote the first and second derivative of $f(x)$, respectively. 
For a function with two arguments $ f(x;y) $, for $ i\in[2] $, $ f_i(x;y) $ denotes the derivative of $ f(x;y) $ with respect to the $i$-th argument. We reserve $\u$ and $\widehat{\u}$ for the true and the MLE of the latent score vector in the pairwise comparison model, respectively. 

We use $G = (V, E)$ to denote an undirected connected graph where $ V $ is the vertices set and $ E $ is the edges set. In this work, we allow the existence of multiple edges in $ E. $ 
For any $U\subset V$, its boundary edges are defined as $\partial U = \{(i,j)\in E: i\in U, j\in U^\complement\}$.
Moreover, for $U_1, U_2\subset V$, we use $\EE(U_1, U_2)$ to denote the set of cross edges between $U_1$ and $U_2$, that is, $\EE(U_1, U_2) = \{(i, j)\in E: i\in U_1, j\in U_2\}$. Specifically, $ \EE(U, U^\complement)  =  \partial U.$


\section{Problem setup}\label{section:setup}
In this section, we introduce a general framework for analyzing pairwise comparison networks.

\subsection{Pairwise comparisons}\label{pcm}
Let $ \u \in \R^n$ be the latent score vector of the subjects of interest. 
In the general framework, the comparison outcome between subjects $i, j\in [n]$ is modeled via a random variable $ X_{ij} $, with density (mass) function given by $f(x; u_i - u_j)$ for some \textit{valid} function $ f $, which is defined below. Note that the distribution of $X_{ij}$ depends only on $u_i-u_j$, which is called the \emph{relative score} between $i$ and $j$.  

\begin{Definition}\label{Pairwise comparison model}
	A function $f: A \times \mathbb{R} \to \mathbb{R^+}, $ where $A$ is a symmetric subset of $\R$ denoting the possible comparison outcomes, is said to be \textit{valid} if it satisfies the following assumptions:
	\begin{Assumption}\label{A1}
	For $y\in\R$, $\int_A f(x;y)\ dx = 1$ if $A$ is continuous, and $\sum_{x\in A}f(x; y) = 1$ if $A$ is discrete.
	\end{Assumption}
	\begin{Assumption}\label{A2}
		$f(x, y)$ is even with respect to $(x; y)$:
		\begin{align*}
		&f(x;y) = f(-x;-y)&(x, y)\in A\times\R.
		\end{align*}
	\end{Assumption}
	\begin{Assumption}\label{A3}
		For $ x < 0 $, $ f(x;y) $ is decreasing in $y$, and $f(x;y)\to 0$ as $ y \to \infty $.
	\end{Assumption}
	
	\begin{Assumption}\label{A4}
		$\sup_{y\in\R}f(x; y)<+\infty$ for every $x\in A$. 
	\end{Assumption}
\end{Definition}

Assumption \ref{A1} guarantees that $\{f(x; y)\}_{y\in\R} $ is a family of probability density (mass) functions indexed by $ y $. Assumption \ref{A2} states that $i$ beats $j$ by $x$ is the same as that $j$ beats $i$ by $-x$. Assumption \ref{A3} implies that a large relative score makes the comparison outcome more predictable. Assumption \ref{A4} is unconditionally true when $A$ is discrete, and is generic for continuous $A$ under appropriate regularity conditions on $f$. 

In accordance with the terms in the comparison data analysis, we say subject $ i $ beats $ j $ if we observe $ X_{ij} > 0.$
If we assign `beat' and `not beat' as a binary relation among the subjects, this binary relation is {\it strongly stochastically transitive} \citep{davidson1959experimental,fishburn1973binary}, that is, if $ \P(X_{ij} > 0)\geq 1/2 $ and $  \P(X_{jk} > 0)\geq 1/2 $, then $$\P(X_{ik} > 0)\geq \max\{\P(X_{ij} > 0), \P(X_{jk} > 0)\}.$$
Stochastic transitivity ensures that latent scores directly translate into good rankings in practice, which in our case is verified under Assumptions \ref{A1}-\ref{A4}.

\begin{prop}[strong stochastic transitivity]
	Under Assumptions \ref{A1}-\ref{A4}, the binary relation `beat` and `not beat' satisfies the strong stochastic transitivity. 
\end{prop}
\begin{proof}
	According to Definition \ref{Pairwise comparison model},
	\begin{align*}
		\P(X_{ij} > 0) = J(u_i-u_j), \text{ where } J(y) = \int_{(0,\infty)} f(x; y) dx.
	\end{align*}
	According to Assumption \ref{A2} and \ref{A3}, $J(y)$ is an increasing function of $y$ with $J(0) \leq 1/2$. 
	Therefore, $\P(X_{ij} > 0)\geq 1/2$ implies $u_i\geq u_j$ for any $ i, j \in [n]. $ The strong stochastic transitivity follows from $ u_i - u_k \geq \max\{u_i - u_j, u_j - u_k  \} $.
\end{proof}



\color{black}
It is worth mentioning that none of the above assumptions requires comparison data to be discrete or continuous. In particular, various choices of $A$ in the literature fit here:
\begin{itemize}
	\item Binary outcome: $ A = \{-1, 1\} $;
	\item Multiple outcome: $ A = \{-k, \ldots,  k\} $ where $k\in\Z^+$ or $ A = \mathbb{Z}; $
	\item Continuous outcome: $ A = [-a, a]$ where $a\in\R^+$ or $ A = \mathbb{R}. $
\end{itemize}

\begin{remark}
	Some works use non-symmetric sets to parametrize the outcomes. For example, in \cite{MR2749821}, a 5-point Likert scale ($A=[5]$) was employed to represent different levels of preference. This is equivalent to $ A = \{-2,-1,0,1,2\}$ as in our case. The symmetry of $ A $ is not special but will make our analysis more convenient and statements more elegant.
\end{remark}
\begin{remark}
	A similar idea appeared in \cite{MR3604649,MR4025736} when considering binary comparison problems.  Specifically, they assume that the distribution of $X_{ij}$ is determined by some symmetric cumulative distribution function $\Phi(t)$:
	\begin{align*}
	\P(X_{ij} = 1) = \Phi(u_i - u_j), \ \P(X_{ij} = -1) = 1 - \Phi(u_i - u_j).
	\end{align*}
	This generalizes the logistic link function in the BT model and can be regarded as a special case under our setup with $ A = \{-1, 1\}$ and $f(1, y) = \Phi(y)$. 
\end{remark}

\subsection{Comparison graphs}
Random structures on comparison graphs could be added to make the framework introduced in Section \ref{pcm} closer to reality. A popular comparison graph structure hypothesizes that the number of comparisons between any pair of items follows $ \Ber(p_n) $ multiplied by some constant $ T $; see \citep{chen2015spectral,jang2016top,MR3613103,MR3953449}. When $ T = 1 $, their comparison graphs are Erd\H{o}s-R\'enyi graphs $ G(n, p_n) .$ The homogeneous assumption makes the analysis easier by avoiding pathological configurations. Nevertheless, statistical results obtained in this vein often have subtle dependence on the assumptions of sparse random graphs that shadow its connection to the graph topology.  


In our framework, we take a slightly different approach by considering the generalized random graph model $ G(n, p_n, q_n) $ as follows:
\begin{Definition}\label{general graph one}
	 $ G(n, p_n, q_n) $ is a random graph with vertices set $ V_n = [n] $ where each edge $(i, j)\in V_n \times V_n$, $i\neq j$ is formed independently with probability $p_{ij,n} \in [p_n, q_n].$
\end{Definition}
According to Definition \ref{general graph one},  $ p_n $ and $ q_n $ could be taken as the minimum and maximum comparison rate, respectively, that is, $	p_{n} := \min_{i,j\in [n]}p_{ij,n}$ and $q_{n} :=  \max_{i,j\in [n]}p_{ij,n}.$
It is well known in \cite{MR0125031} that $ G(n, p_n) $ is disconnected with high probability if $ p_n < (1-\epsilon) n^{-1}\log n$, for any constant $\epsilon>0$. As a result, there exist at least two connected components which cannot be estimated together unless additional constraints are imposed. 
Therefore, the connectivity threshold $n^{-1}\log n $ is the best possible lower bound on $p_n$ that one can hope for. 
Since $ G(n, p_n, q_n) $ reduces to the Erd\H{o}s-R\'enyi graph  $G(n, p_n)$ when $p_n = q_n$, it shares the same optimal lower bound on $ p_n $.

Given two subjects may have multiple comparisons between them, we assume that the number of comparisons between $ i $ and $ j $, $ n_{ij}$, satisfies $n_{ij} \sim \text{Bin}(T, p_{ij,n}). $ 
One can also adopt the setup in \cite{MR3953449} by assuming $ n_{ij} \sim T\times \Ber(p_{ij,n}) $ and the proof is similar.  
Since $ T $ is a fixed constant,  the reader could take $ T = 1 $ to avoid the additional multiplicative constant for ease of understanding.


\begin{remark}
	We adopt $ G(n, p_n, q_n) $ as the comparison graph when deriving the entry-wise error in Section \ref{section:main}. 
	Compared to the Erd\H{o}s-R\'enyi graph, our choice takes into account the potential heterogeneity of degree sequences. 
	Our analysis applies to a much wider class of comparison graphs beyond random graph models. A more comprehensive discussion on comparison graph structure will be carried out in Section \ref{Comparison graph}. It is worth emphasizing that we require the comparison graph does not depend on the latent score $ \bu $ in the whole paper.  

\end{remark}


\color{black}


\section{Main results}\label{section:main}

We consider estimating $\u$ via the maximum likelihood estimator (MLE) under the general framework introduced in Section \ref{section:setup}. 
Let $\{n_{ij}\}_{i,j\in [n], i\neq j}$ be the number of comparisons observed between subjects $i$ and $j$. $\{X^{(t)}_{ij}\}_{t\in [n_{ij}]}$ denote the outcomes between $i$ and $j$ in $n_{ij}$ comparisons. 
The conditional log-likelihood function given $\{n_{ij}\}_{i,j\in [n], i\neq j}$ is 
\begin{equation*}
l(\v) = \frac{1}{2}\sum_{i, j\in [n]} \sum_{t\in [n_{ij}]} \log f(X_{ij}^{(t)};v_i - v_j),
\end{equation*}
where $\v = (v_1, \cdots, v_n)^T\in\R^n$. 
Since $ l(\v) = l(\v + z \mathbf{1}) $ for $ z \in \mathbb{R}, $ an additional constraint on the parameter space is required to make $\u$ identifiable. 
For this we take $v_1 = 0$, see \citep{MR3012434}. 
The MLE, $\widehat{\u}$ is defined as 
\begin{equation}
\widehat{{\u}} = \argmax_{\v \in \mathbb{R}^n, v_1 = 0}l(\v).\label{mle:formula}
\end{equation}
Note $\eqref{mle:formula}$ only holds formally unless it admits a unique maximizer. 
We will show that, for sufficiently large $n$, $l(\v)$ attains a unique maximum under an appropriate condition on $G(n, p_n, q_n)$.  

\subsection{Existence and uniqueness}
Condition \ref{Condition: existence} below provides a sufficient condition that ensures the unique existence of the MLE defined in \eqref{mle:formula}.

\begin{condition}\label{Condition: existence}
Denote by $G = (V, E)$ the comparison graph where $ V $ is the set of vertices and $ E $ is the set of edges. 
For any non-empty subset $V_1\subsetneq V$, there exist $i\in V_1$, $j\in V\setminus V_1$ and $t\in [n_{ij}]$ such that $X^{(t)}_{ij}>0$. 
In other words, for every partition of $ V, $ into two nonempty sets, a subject in the first set has defeated a subject in the second at least once.
\end{condition}
Note that Condition \ref{Condition: existence} is well known in the BT model \citep{zermelo1929berechnung, ford1957solution}. We show Condition \ref{Condition: existence} is also enough to ensure the unique existence of $ \widehat{{\u}} $ in the generalized model, which is stated in the following lemma.

\begin{Lemma}\label{exist1}
	Under Condition \ref{Condition: existence}, $\widehat{{\u}}$ defined in \eqref{mle:formula} uniquely exists. 
\end{Lemma}

\begin{proof}
See Appendix \ref{eeee}. 
\end{proof}

Next, we demonstrate the Condition $ \ref{Condition: existence} $ holds almost surely under certain assumption. The \emph{dynamic range} of $\u$ plays a key role in our analysis and is defined by 
$$M_n := \max_{i,j}|u_i - u_j|.$$
According to Assumption \ref{A3}, the \emph{global discrepancy} between different subjects is defined by 
\begin{equation*}
C_n^{(1)}  = \int_{[0,\infty)} f(x;M_n) \ dx. 
\end{equation*}
One may think of global discrepancy as an objective measure for the difference between abilities of the subjects. It is easy to verify that $ C_n^{(1)}$ is increasing in $M_n$, and $ C_n^{(1)} \in [1/2,1]$.  
Moreover, we need the following definition to characterize the log-likelihood function:
\begin{Definition}[log-concavity]
	A positive function $h(x;y)$ is said to be log-concave with respect to $y$ if  for every $x$, 
	\begin{align*}
		\frac{\partial^2}{\partial y^2}\log h(x;y)\leq 0, \ \ \forall y\in\R,
	\end{align*}
	and is strictly log-concave with respect to $y$ if the strict inequality holds. 
\end{Definition}

Note that this definition automatically assumes $\log h(x; y)$ is twice differentiable with respect to $y$, and such default regularity assumption will make the following analysis more straightforward. 

\begin{thm}[unique existence of the MLE]\label{exist}
	Suppose that $ f(x;y) $ is strictly log-concave with respect to $ y $, and 
	\begin{align}\label{condition:existence}
		\frac{\log n}{np_{n}\left |\log C_n^{(1)}\right |}\to 0 \text{ as } n \to \infty. 
	\end{align}
	Then for sufficiently large n, with probability at least $ 1 - n^{-3} $, Condition \ref{Condition: existence} holds. Therefore,
	with probability 1, $ \widehat{\u} $ uniquely exists for all but finitely many $ n $.
\end{thm}

\begin{proof}
See Appendix \ref{lkg}. 
\end{proof}

\begin{remark}
	When $\lim_{n\to\infty}C_n^{(1)}<1$, {\eqref{condition:existence} is satisfied if there exists $ \epsilon > 0 $ such that $p_n\gtrsim n^{-1}(\log n)^{1+\epsilon}$, which is close to the lower bound of $ p_n $.} 
	Moreover, as $M_n$ measures the spread of $\u$ and $ p_n $ is the minimum edge density in a network, either large $ M_n $ or small $ p_n $ may result in an increasing chance of Condition \ref{Condition: existence} being violated. 
	Since $ C_n^{(1)} \to 1 \text{~as~} M_n \to \infty$, this observation is codified in \eqref{condition:existence}.
\end{remark}

\subsection{Uniform consistency for bounded $ A $}\label{ucb}
We start with the case when $ A $ is a bounded set. Let 
$$ g(x;y) = \frac{\partial }{\partial y}\log f(x; y) = \frac{f_2(x;y)}{f(x;y)}.$$  Define two additional constants depending only on  $ M_n $ and $ g(x; y)$: 
\begin{eqnarray*}
	C_n^{(2)} := \sup\limits_{x\in A, |y|\leq M_n} |g(x;y)|,\	C_n^{(3)} := \inf\limits_{x\in A, |y|\leq M_n+1} |g_2(x;y)|.
\end{eqnarray*}
The ratio of  the above two constants, 
\begin{align}
 \omega_n := \frac{C_n^{(2)}}{C_n^{(3)}}\label{myomega}
 \end{align}
  is used in formulating a condition that ensures the uniform consistency of the MLE in the following theorem:

\begin{thm}[uniform consistency, bounded case]\label{main}
	Suppose that $ f(x;y) $ is strictly log-concave with respect to $ y $, and
	\begin{equation}\label{mle:pn}
		\Delta_n :=  \omega_n\sqrt{\frac{q_n^2(\log n)^3}{np_n^3}} \rightarrow 0\ \ \ \text{as}\ n\rightarrow\infty,
	\end{equation}
	where $\omega_n$ is defined in \eqref{myomega}.
		If \eqref{condition:existence} holds true, then there exists an absolute constant $C>0$, such that for sufficiently large $n$, with probability at least $1-n^{-2}$, $\widehat{\u}$ uniquely exists and satisfies
		\begin{equation*}
			\left\lVert \widehat{{\u}} - \u\right\rVert_\infty\leq C\Delta_n.
		\end{equation*}
		In other words, $\widehat{\u}$ is a uniformly consistent estimator for $\u$. 
\end{thm}

\begin{proof}
See Appendix \ref{proof:ASC}.
\end{proof}


Although we target at a general graph $ G(n, p_n, q_n) $, it is interesting to obtain a simplified lower bound on $ p_n $ when $q_n\lesssim p_n$ and $\max\{M_n, \omega_n\} \lesssim 1$:

\begin{cor}\label{cor1}
	Suppose $\max\{M_n, \omega_n, q_n/p_n\}\lesssim 1$. If there exists $ \epsilon > 0 $ such that
	\begin{equation}\label{cor1: eq}
		p_n \gtrsim\frac{(\log n)^{3+\epsilon}}{n},
	\end{equation}
	then $\widehat{\u}$ uniquely exists a.s. for all but finitely many $n$, and is uniformly consistent for $\u$ with convergence rate at least of order $(\log n)^{-\epsilon/2}$ . 
\end{cor}

\begin{remark}
	When $A$ is bounded, $\omega_n <\infty$ can be easily satisfied by imposing some regularity conditions on $g$. In fact, if both $g(x;y)$ and $g_2(x; y)$ are continuous functions on $A\times\R$, or $g_2(x; y)$ is a continuous function in $y$ and $A$ is a finite set, one can deduce that $C_n^{(2)} <\infty$ and $C_n^{(3)}>0$, which follows from the fact that $A\times [-M_n-1, M_n+1]$ is contained in a compact set in $\R^2$ and the strict log-concavity assumption on $f$. 
\end{remark}

\begin{remark}
	Corollary \ref{cor1} gives a lower bound on $ p_n $ which only differs from the theoretical possible lower bound $ n^{-1}\log n$ by a logarithmic factor, implying that \eqref{cor1: eq} stated in Corollary \ref{cor1} is almost optimal.  The effective lower bound in our result is $ n^{-1}(\log n)^3 $, and the larger the $ p_n $, the faster the convergence rate.
\end{remark}
\color{black}

\subsection{Uniform consistency for general $A$}
We now consider the case when $ A $ is unbounded. Note that Theorem \ref{main} is vacuous unless $ C_n^{(2)} < \infty$ and $C_n^{(3)} > 0   $, which can be easily verified under proper regularity assumptions when $A$ is bounded. As we will see next, the same remains true when $A$ is unbounded except for $C_n^{(2)}<\infty$, which requires an extra condition to be imposed.

\begin{Lemma}\label{exp}
	Let $f(x; y)=h(x)\exp\left(yT(x)-a(y)\right)$ be an exponential family, where $y$ is the natural parameter and $T(x)$ is the sufficient statistic. Let $X_y$ be a random variable whose distribution is $f(x;y)$. Suppose that 
	$\mathcal{V}(y):=\text{Var}[T(X_y)]>0$ is a continuous function in $y$. 
	Then $f(x;y)$ is strictly log-concave, and  $C_n^{(3)}>0$. 
\end{Lemma}

\begin{proof}
	It is easy to see from direct computation that 
	\begin{equation*}
		g_2(x;y)= -a''(y)=-\text{Var}[T(X_y)]=-\mathcal{V}(y).
	\end{equation*} 
	Since $\mathcal{V}(y)$ is continuous in $y$ and positive everywhere, $g_2(x;y)<0$ for every $y$, proving that $f(x;y)$ is \emph{strictly log-concave} in $y$, and $ C_n^{(3)} = \min_{|y|\leq M_n+1}\mathcal{V}(y)>0. $
\end{proof}
Note that under the same condition as in Lemma \ref{exp}, 
\begin{equation*}
	C_n^{(2)} = \max_{x\in A, |y|\leq M_n} |T(x) - a'(y)|=\max_{x\in A, |y|\leq M_n} |T(x) - \E[T(X_y)]|,
\end{equation*} 
which in general is unbounded. For example, the family of normal distributions with unit variance has $C_n^{(2)}=\sup_{x>0, |y|\leq M_n}\left |x-y\right |=\infty$. To fix this issue, we introduce the following uniform sub-gaussian condition as an alternative requirement:  


\begin{condition}\label{con-2}
	$\{g(X_y, y)-\E[g(X_y, y)]\}_{|y|\leq M_n}$, as a sequence of random variables indexed by $y$, has a uniformly bounded sub-gaussian norm, that is, there exists some constant $C_n^{(4)}$ such that
	\begin{align*}
		&\sup_{|y|\leq M_n}\|g(X_y, y)-\E[g(X_y, y)]\|_{\psi_2} \\
		:=&\ \sup_{|y|\leq M_n}\inf\left\{t\geq 0: \E\left[\exp\left(\frac{(g(X_y, y)-\E[g(X_y, y)])^2}{t^2}\right)\right]\leq 2\right\}\leq C_n^{(4)}. 
	\end{align*}
\end{condition}

Intuitively, $C_n^{(4)}$ in Condition \ref{con-2} plays an analogous role as $C_n^{(2)}$ in the case of bounded $A$. 
One could also use a uniform sub-exponential condition to further generalize Condition \ref{con-2} by replacing the sub-gaussian norm $ \| \cdot\|_{\psi_2}$ with the sub-exponential norm $ \| \cdot\|_{\psi_1}$.
As the ideas are similar, we will only focus on Condition \ref{con-2} in this article for ease of presentation.   
In particular, the sub-gaussian assumption is sufficient to include the normal distribution as desired. 
With slight abuse of notation, we let 
\begin{align}
\omega_n = \min\left\{\frac{C_n^{(4)}}{C_n^{(3)}},\frac{C_n^{(2)}}{C_n^{(3)}}\right\}.\label{myomega1}
\end{align}
The following theorem establishes the uniform consistency of the MLE for general $ A $:

\begin{thm}[Uniform consistency of the MLE, general case]\label{general}
	Suppose that $ f(x;y) $ is strictly log-concave with respect to $ y $. Suppose that Condition \ref{con-2} holds and
	\begin{equation*}
	\Delta_n :=  \omega_n\sqrt{\frac{q_n^2(\log n)^3}{np_n^3}} \rightarrow 0\ \ \ \text{as}\ n\rightarrow\infty, 
\end{equation*}
where $\omega_n$ is defined in \eqref{myomega1}.
		If \eqref{condition:existence} holds true, then there exists an absolute constant $C>0$, such that for sufficiently large $n$, with probability at least $1-n^{-2}$, $\widehat{\u}$ uniquely exists and satisfies
		\begin{equation*}
			\left\lVert \widehat{{\u}} - \u\right\rVert_\infty\leq C\Delta_n.
		\end{equation*}
		In other words, $\widehat{\u}$ is a uniformly consistent estimator for $\u$. 
\end{thm}

\begin{proof}
See Appendix \ref{proof:ASC}.
\end{proof}

\color{black}

\begin{remark}
	Although Theorem \ref{general} may seem restrictive, it covers many common unbounded models of interest, that is, 
	the exponential family parametrization mentioned in Lemma \ref{exp} with the sub-gaussian property. The example of normal distribution is given in the following section. 
\end{remark}

\color{black}

\section{Examples}\label{section:examples}
In this section, we demonstrate that several well-known parametric pairwise comparison models can be covered within our setup. 
Particularly, the uniform consistency of the MLE can be proved for all these models provided $p_n$ is reasonably sparse. To the best of our knowledge, except for the Bradley-Terry model, the uniform consistency property of all the other models in this section is new.  

\subsection{Binary outcomes} \label{panger} 
We first consider the Bradley-Terry (BT) model and the Thurstone-Mosteller model \citep{thurstone1927law,mosteller2006remarks}. In both cases, $ A = \{-1,1\} $ and $ f(x;y) $ is given by
\begin{align*}
	&f(1;y) = \Phi(y) &\ f(-1;y) = 1 -  \Phi(y),
\end{align*} 
with 
\begin{align*}
 \Phi(y)  &= \frac{e^{y}}{1+e^{y}}, &\text{(Bradley-Terry model)},\\
 \Phi(y) & = \frac{1}{\sqrt{2\pi}}\int_{-\infty}^y e^{-\frac{x^2}{2}}dx,&\text{(Thurstone-Mosteller model)},
\end{align*}
It is easy to check that in both cases $f(x; y)$ is {\it valid} and {\it strictly log-concave} with respect to $y$. 
According to Theorem \ref{main}, we have the following results:
\begin{cor}[Bradley-Terry model]\label{BT_model}
	In the Bradley-Terry model, if $q_n, p_n$ and $ M_n $ satisfy the following bound:
	\begin{equation*}
		\Delta_n =  e^{M_n}\sqrt{\frac{q_n^2(\log n)^3}{np_n^3}} \to 0\ \ \ \text{as~} n \to \infty, 
	\end{equation*}
	then for sufficiently large $n$, with probability at least $1-n^{-2}$, $\widehat{\u}$ uniquely exists and satisfies $ \left\lVert \widehat{{\u}} - \u\right\rVert_\infty\lesssim \Delta_n $. In particular, when $M_n = \Omega(1)$ and taking $ G(n, p_n, q_n) $ as the Erd\H{o}s-R\'enyi graph, $p_n\gtrsim n^{-1}(\log n)^{3 + \epsilon}$ for some $ \epsilon > 0 $ is sufficient for the uniform consistency of $\widehat{\u}$, namely, $ \left\lVert \widehat{{\u}} - \u\right\rVert_\infty\lesssim (\log n)^{-\epsilon/2}.$ 
\end{cor}

\begin{cor}[Thurstone-Mosteller model]\label{TM_model}
	In the Thurstone-Mosteller model, if $q_n, p_n$ and $ M_n $ satisfy
	\begin{align*}
\Delta_n =  e^{M_n^2/2}\sqrt{\frac{q_n^2(\log n)^3}{np_n^3}} \to 0 \ \ \ \text{as~} n \to \infty, 
	\end{align*}
	then for sufficiently large $n$, with probability at least $1-n^{-2}$, $\widehat{\u}$ uniquely exists and satisfies $ \left\lVert \widehat{{\u}} - \u\right\rVert_\infty\lesssim \Delta_n $. In particular, when $M_n = \Omega(1)$ and taking $ G(n, p_n, q_n) $ as the Erd\H{o}s-R\'enyi graph, $p_n\gtrsim n^{-1}(\log n)^{3 + \epsilon}$ for some $ \epsilon > 0 $ is sufficient for the uniform consistency of $\widehat{\u}$, namely, $ \left\lVert \widehat{{\u}} - \u\right\rVert_\infty\lesssim (\log n)^{-\epsilon/2}. $
\end{cor}

In both models, condition \eqref{mle:pn} implies condition \eqref{condition:existence}. Corollary \ref{TM_model} gives the first entry-wise error analysis of the MLE in the Thurstone-Mosteller model while Corollary \ref{BT_model} provides the entry-wise error bound for the MLE in the BT model under a general comparison graph. 
It is worth pointing out that the uniform consistency of the MLE in the BT model has been well studied under the structure of the Erd\H{o}s-R\'enyi graphs. 
Although we focus on comparison models with a more general structure, our result is effectively the same as \cite{BTL} and only slightly worse than \cite{chen2020partial} up to logarithmic factors. 	With some technical refinement of our proof, we can exactly recover the result in \cite{BTL}.


\begin{remark}
Both our work and \cite{BTL} utilize the idea of chaining to deal with the entry-wise error, which builds an upward nested sequence of vertices at the ends of the estimation error spectrum. 
However, a direct application of the proof in \cite{BTL} in our case will lead to an overlapping scenario that is difficult to analyze. 
Instead, our new proof exploits a symmetry property (Assumption \ref{A2}) to avoid technical difficulty. 
\end{remark}


\color{black}

\subsection{Multiple outcomes} 
\cite{agresti1992analysis} provided two generalizations the BT model by taking account of multiple outcomes (ordinal data): one is the cumulative link model and the other is the adjacent categories model. 
It can be verified that both models satisfy the validity and strict log-concavity assumptions which lead to the uniform consistency of the corresponding MLE under our framework. 
Specifically, when there are only three outcomes (`win', `tie', `loss', or $A=\{-1, 0, 1\}$), the cumulative link model and the adjacent categories model reduce to the Rao-Kupper model \citep{MR217963} and the Davidson model \citep{davidson1970extending}, respectively. 
In this section, we prove the uniform consistency of the MLE for these two models. 

The link function $ f(x;y) $ in the Rao-Kupper model is given by
\begin{align*}
f(1;y) = \frac{e^y}{e^y + \theta}; \ f(0;y) =\frac{(\theta^2 - 1)e^y}{(e^y + \theta)(\theta e^y + 1)}; \
f(-1;y) =\frac{1}{\theta e^y + 1},
\end{align*}
where $ \theta > 1 $ is the threshold parameter which is assumed to be fixed. The following corollary is straightforward from Theorem \ref{main}.
\begin{cor}[Rao-Kupper model]\label{RK model}
	In the Rao-Kupper model, for fixed $ \theta >1$, if $q_n, p_n$ and $ M_n $ satisfy
	\begin{equation*}
	\Delta_n =  e^{M_n}\sqrt{\frac{q_n^2(\log n)^3}{np_n^3}} \to 0\ \ \ \text{as~} n \to \infty, 
\end{equation*}
then for sufficiently large $n$, with probability at least $1-n^{-2}$, $\widehat{\u}$ uniquely exists and satisfies $ \left\lVert \widehat{{\u}} - \u\right\rVert_\infty\lesssim \Delta_n $. In particular, when $M_n = \Omega(1)$ and taking $ G(n, p_n, q_n) $ as the Erd\H{o}s-R\'enyi graph, $p_n\gtrsim n^{-1}(\log n)^{3 + \epsilon}$ for some $ \epsilon > 0 $ is sufficient for the uniform consistency of $\widehat{\u}$, namely, $ \left\lVert \widehat{{\u}} - \u\right\rVert_\infty\lesssim (\log n)^{-\epsilon/2}. $
\end{cor}

As opposed to the Rao-Kupper model, the Davidson model \citep{davidson1970extending} considers an alternative outcome of being tie: 
\begin{align*}
	f(1;y) = \frac{e^y}{e^y + \theta e^{\frac{y}{2}} + 1}; \  f(0;y) = \frac{\theta e^{\frac{y}{2}}}{e^y + \theta e^{\frac{y}{2}} + 1};\
	f(-1;y) = \frac{1}{e^y + \theta e^{\frac{y}{2}} + 1},
\end{align*}
where $ \theta > 0 $ is assumed to be fixed. 
Similarly, we have the following corollary. 
\begin{cor}[Davidson model]\label{D_model}
In the Davidson model, for fixed $ \theta >0$, if $q_n, p_n$ and $ M_n $ satisfy
	\begin{equation*}
	\Delta_n =  e^{M_n/2}\sqrt{\frac{q_n^2(\log n)^3}{np_n^3}} \to 0\ \ \ \text{as~} n \to \infty, 
\end{equation*}
then for sufficiently large $n$, with probability at least $1-n^{-2}$, $\widehat{\u}$ uniquely exists and satisfies $ \left\lVert \widehat{{\u}} - \u\right\rVert_\infty\lesssim \Delta_n $. In particular, when $M_n = \Omega(1)$ and taking $ G(n, p_n, q_n) $ as the Erd\H{o}s-R\'enyi graph, $p_n\gtrsim n^{-1}(\log n)^{3 + \epsilon}$ for some $ \epsilon > 0 $ is sufficient for the uniform consistency of $\widehat{\u}$, namely, $ \left\lVert \widehat{{\u}} - \u\right\rVert_\infty\lesssim (\log n)^{-\epsilon/2}. $
\end{cor}

Similar to the examples in the previous section, condition \eqref{mle:pn} implies condition \eqref{condition:existence} for both the Rao-Kupper model and the Davidson model.
Utilizing a similar argument, the uniform consistency result can be proved for the extensions of the Thurstone-Mosteller model \citep{maydeu2001limited, maydeu2002limited}, which we do not state here.


\subsection{Continuous outcomes}
We consider the paired cardinal  model  introduced and studied in \cite{MR3504618}. 
In this case, $ A = \mathbb{R}$, and $ X_{ij} $ follows a normal distribution with mean $u_i-u_j$ and variance $\sigma^2$, that is, $X_{ij}\sim \mathcal N(u_i - u_j, \sigma^2) $ and 
\begin{align*}
f(x;y) = \frac{1}{\sqrt{2\pi \sigma^2}}e^{-\frac{(x - y)^2}{2\sigma^2}}, \ x\in\R.
\end{align*}
Cardinal models take a continuous spectrum of measurement values and often contain more information than ordinal models under certain conversion assumptions \citep{MR3504618}. 
Particularly, the Thurstone-Mosteller model studied in Section \ref{panger} can be viewed as a thresholded version of the paired cardinal model; the binary outcome between $i$ and $j$ corresponds to the cases where $X_{ij}>0$ and $X_{ij}\leq 0$. 

Validity and log-concavity assumptions on $f$ are easy to verify for the paired cardinal model. 
To apply Theorem \ref{general}, we have
\begin{cor}[Paired cardinal model]\label{Normal distribution model}
	In the paired cardinal model, for fixed $ \sigma>0$, if $q_n$, $p_n$ and $ M_n $ satisfy
	\begin{equation*}
M_n	e^{\frac{M_n^2}{2\sigma^2}}\frac{\log n}{np_n} \to 0\ \ \ \text{as~} n \to \infty, 
	\end{equation*}	
	and
	\begin{equation*}
\Delta_n =  \sqrt{\frac{q_n^2(\log n)^3}{np_n^3}} \to 0\ \ \ \text{as~} n \to \infty,
	\end{equation*}
then for sufficiently large $n$, with probability at least $1-n^{-2}$, $\widehat{\u}$ uniquely exists and satisfies $ \left\lVert \widehat{{\u}} - \u\right\rVert_\infty\lesssim \Delta_n $. In particular, when $M_n = \Omega(1)$ and taking $ G(n, p_n, q_n) $ as the Erd\H{o}s-R\'enyi graph, $p_n\gtrsim n^{-1}(\log n)^{3 + \epsilon}$ for some $ \epsilon > 0 $ is sufficient for the uniform consistency of $\widehat{\u}$, namely, $ \left\lVert \widehat{{\u}} - \u\right\rVert_\infty\lesssim (\log n)^{-\epsilon/2}. $ 
\end{cor}
\section{Comparison graph structure}\label{Comparison graph}
In this section, we summarize two topological conditions of a graph that will be used in the consistency analysis. In particular, we first assume that comparison graphs are deterministic and enjoy certain topological properties; we then
verify these properties in several classes of random graph models of practical interest. 
Our approach decouples the randomness in the comparison model and comparison graph thus is more transparent in terms of illustrating how graph topology influences the statistical procedures in pairwise data analysis.

Some definitions and notations in the graph theory are introduced here \citep{MR1421568}. Let $ G = (V,E) $ be an undirected graph. For $U\subset V$, the connectivity between $U$ and $U^\complement$ can be measured by the following ratio cut $h_{G}(U)$:
\begin{align*}
h_G(U) = \frac{|\partial U|}{\min\{|U|, |U^\complement|\}}.
\end{align*}
A large value of $h_{G}(U)$ suggests that $U$ and $U^\complement$ are well-connected. The global connectivity of $G$ can be measured by taking the minimum of $h_{G}(U)$ over $U$,
\begin{align*}
h_G  = \min_{U\subset V} h_G(U),
\end{align*}
which is called the \emph{modified Cheeger constant} or the \emph{isoperimetric number}, of $G$.

In the analysis of comparison graph models, we are concerned with the asymptotic behavior of the model, which requires us to work with a sequence of graphs. 
For convenience, we let $\{G_n\}_{n\in\N}$ denote a graph sequence where $G_n = (V_n, E_n)$ with $|V_n| = n$. 
We now introduce two topological properties that are central to our analysis:
\begin{Definition}[ASC]\label{def3}
	A sequence of graphs $\{G_n\}_{n\in\N}$ is said to be asymptotically strongly connected (ASC) with rate $\{\omega_n\}_{n\in\N}$ if  
	\begin{align*}
	\lim_{n\to\infty}\omega_n\Gamma_{G_n}^{ASC}= {0}, \text{ where  }\Gamma_{G_n}^{ASC} := \sqrt{\frac{\log n}{h_{G_n}}}.
	\end{align*}
\end{Definition}

\begin{Definition}[RE]\label{def2}
	For $G = G(V, E)$, an upward nested sequence of non-empty vertices set $\{A_k\}_{k=1}^K$ 
	(that is $ A_k\subsetneq A_{k+1}\subseteq V$)
	 is called admissible if $|\EE(A_k, A_{k+1}\setminus A_k)|\geq |\partial A_k|/2$ for all $k<K$.   
	Denote the set of admissible sequences of $G$ as $\mathscr A(G)$. 
	$\{G_n\}_{n\in\N}$ is called rapidly expanding (RE) with rate $\{\omega_n\}_{n\in\N}$ if 
	\begin{align*}
\lim_{n\to\infty}\omega_n\Gamma_{G_n}^{RE}={0}, \text{ where  }\Gamma_{G_n}^{RE} := \max_{\{A_k\}_{k=1}^K\in \mathscr A(G_n)}\sum_{k=1}^{K - 1}\sqrt{\frac{\log n}{h_{G_n}(A_k)}}.
	\end{align*}
\end{Definition}

ASC is a global property that requires ``small'' subsets of $G_n$ have relatively large edge boundary as $n\to\infty$, and RE is a cumulative version of ASC (defined for all rapidly expanding sequences).   
It is easy to verify that RE with rate $\{\omega_n\}_{n\in\N}$ implies $ ASC $ with the same rate by taking any admissible sequence with $ A_1 $ that satisfies $ h_{G_n}(A_1) = h_{G_n}$. 
Note that an admissible sequence $ \{A_k\}_{k=1}^K$ is strictly increasing by definition, therefore there exists a natural upper bound for $ K $, that is, $ K \leq n $.


A sufficient condition for the uniform consistency result can be formulated using RE, with the convergence rate of $ \Gamma_{G_n}^{RE}  $ (that is,  $\omega_n$) appropriately chosen to encode the information of the pairwise comparison parametrization. The uniform consistency result for the MLE can be stated as follows:

\begin{thm}[uniform consistency]\label{main: general}
	Suppose $ f(x;y) $ is strictly log-concave with respect to $y$ and Condition \ref{Condition: existence} and \ref{con-2} hold.
	If the comparison graph sequence $\{G_n\}_{n\in\N}$ is RE with rate $\{\omega_n\}_{n\in\N}$, where $\omega_n$ is defined in \eqref{myomega1}, that is,
		\begin{align*}
	\Delta_n^{RE}:= \omega_n\Gamma_{G_n}^{RE}  \to 0 \ \text{as} \ n \to \infty, 
	\end{align*}
 then there exists an absolute constant $C>0$, such that for sufficiently large $n$, with probability at least $1-n^{-2}$, $\widehat{\u}$ uniquely exists and satisfies
\begin{equation}\label{eq: main: general}
\left\lVert \widehat{{\u}} - \u\right\rVert_\infty\leq C\Delta_n^{RE}.
\end{equation}
In other words, $\widehat{\u}$ is a uniformly consistent estimator for $\u$. 
\end{thm}

\begin{proof}
See Appendix \ref{section:proof}.
\end{proof}

\begin{remark}
Theorem \ref{main: general} requires the comparison graph sequence $\{G_n\}_{n\in\N}$ is RE. The requirement $|\EE(A_k, A_{k+1}\setminus A_k)|\geq |\partial A_k|/2$ in the definition of admissible sequences could be changed to $ |\EE(A_k, A_{k+1}\setminus A_k)|\geq |\partial A_k|/(1+\epsilon) $ for any absolute constant $ \epsilon > 0$. In that case, the result in Theorem 4 remains unchanged up to a multiplicative constant in \eqref{eq: main: general}.
	
\end{remark}

Theorem \ref{main: general} provides the uniform consistency of the MLE in the generalized comparison model under a general comparison graph. 
The convergence rate consists of two parts: $ \omega_n $ and $ \Gamma_{G_n}^{RE}  $. $ \omega_n $ is defined as $ \min\{C_n^{(4)}, C_n^{(2)}\}/C_n^{(3)} $ , which relies only on the comparison model $ f $ and the dynamic range of $ \bu $. $ \Gamma_{G_n}^{RE} $ is concerned with the topological property of the comparison graph. 
Compared with Theorem \ref{general}, Theorem \ref{main: general} replaces $ \Delta_n $ with $ \Delta_n^{RE}. $ 
Consequently, Theorem \ref{general} implies Theorem \ref{main: general} if the $ \Gamma_{G_n}^{RE} $ in $ G(n, p_n, q_n) $ is bounded by $ \sqrt{{q_n^2(\log n)^3}/{np_n^3}} $.
To further demonstrate the utility of Theorem \ref{main: general}, we also prove that RE is satisfied in
the stochastic block model \citep{holland1983stochastic} with an additional structure.
These results are summarized in the following proposition: 

\begin{prop}\label{random graph ensembles}\label{myprop}
	Let $\{G_n\}_{n\in\N}$ be a (random) graph sequence. For all sufficiently large $n$, the following events hold with probability at least $1-n^{-2}$, 
	\begin{itemize}
		\item[1.] If $ G_n = G(n, p_n) $, then $ \Gamma_{G_n}^{RE} \lesssim  \sqrt{{(\log n)^3}/{np_n}}$.
		\item[2.] If $ G_n = G(n, p_n, q_n) $ (see Definition \ref{general graph one}), then $ \Gamma_{G_n}^{RE} \lesssim  \sqrt{{q_n^2(\log n)^3}/{np_n^3}}$.
		\item[3.] If $ G_n $ is a stochastic block model with finite blocks and  its lower bound of edge density is $ p_n $, then $ \Gamma_{G_n}^{RE} \lesssim  \sqrt{{(\log n)^3}/{np_n}}$.
	\end{itemize}
\end{prop}

\begin{proof}
See Appendix \ref{proof:ASCC}.
\end{proof}

\color{brickred}
\color{black}

\section{Numerical results}\label{section:num}
In this section, we first conduct numerical simulations to evaluate the large-sample performance of the MLE in the Davidson model with threshold parameter $\theta=1$ and the paired cardinal model with variance parameter $\sigma^2=1$. Since extensive numerical results exist for both models using real datasets \citep{MR3504618, MR3887567}, our simulations are more focused on the synthetic data, which mainly serve to verify the asymptotic results in Sections \ref{section:main} and \ref{section:examples}.
The corresponding results are reported in Section \ref{sec:6.1}. 
Moreover, as our framework provides ample flexibility for model parametrization, it is tempting to test different model parametrizations on a dataset and select the optimal one for use in practice using model selection methods. 
We empirically investigate this problem on a real dataset in Section \ref{han}. 



\subsection{Asymptotic performance}\label{sec:6.1}



We first test the asymptotic uniform convergence of the MLE when the network is sparse. 
Note that a large value of $T$ can inadvertently make the network dense even if $p_n$ is small. As such, we set $T=1$ in the following simulations. 
The comparison graph model is set as $ G(n, p_n, q_n)$, with its size chosen in an increasing manner to demonstrate the expected convergence.
Specifically, we test on $6$ different values for $n$: $2000, 4000, 6000, 8000, 10000$ and $12000$. For each $n$, the latent score vector $\u$ is generated by independently sampling its components from the uniform distribution on $[-0.5, 0.5]$, which guarantees that $M_n\leq 1$. The minimum comparison rate $p_n$ is taken as $n^{-1}(\log n)^3, \sqrt{n^{-1}(\log n)^3}$ and $1/2$, corresponding to the underlying network being sparse, moderately sparse and dense, respectively. For convenience, we let the maximum comparison rate $ q_n $ is proportional to $ p_n $, that is, $ q_n = 2p_n. $ Values of $ p_n $ under different $ n $ are presented in Table \ref{pn_mn}.


\begin{table}
	\centering
	\begin{tabular}{|c|c|c|c|c|}
		\hline
		$n$     & $p_n =  \sqrt{n^{-1}(\log n)^3} $         & $ p_n =  {n^{-1}(\log n)^3} $         & $ M_n = \log \log n/2 $    & $M_n = 2\log \log n$    \\ \hline
		2000  & 0.469(937) & 0.220(439) & 1.014 & 4.057 \\ \hline
		4000  & 0.378(1511) & 0.143(571) & 1.058 & 4.231 \\ \hline
		6000  & 0.331(1988) & 0.110(658) & 1.082 & 4.327 \\ \hline
		8000  & 0.301(2410) & 0.091(726) & 1.098 & 4.392 \\ \hline
		10000 & 0.280(2795) & 0.078(781) & 1.110 & 4.441 \\ \hline
		12000 & 0.263(3153) & 0.069(829) & 1.120 & 4.480 \\ \hline
	\end{tabular}
	\caption{\small The value of $ p_n $ and $ M_n $ given the different $ n $. In addition,  the average numbers of comparisons one subject has (in parentheses) is in the column of $ p_n $.}
	\label{pn_mn}
\end{table} 

For every fixed $n, p_n$ and $\u$,  the comparison data is generated under the respective model with $\widehat{\u}$ computed using a minorization--maximization (MM) algorithm in \cite{MR2051012}. We then calculate the $\ell_\infty$ error $\|\widehat{\u}-\u\|_\infty$. To check uncertainty, for each $n$ and $p_n$, the experiment is repeated $300$ times with its quartiles recorded. The results are reported in the first two plots in Figure \ref{davidson}. 


\begin{figure}[htbp]
	\centering
	\includegraphics[width=0.45\textwidth]{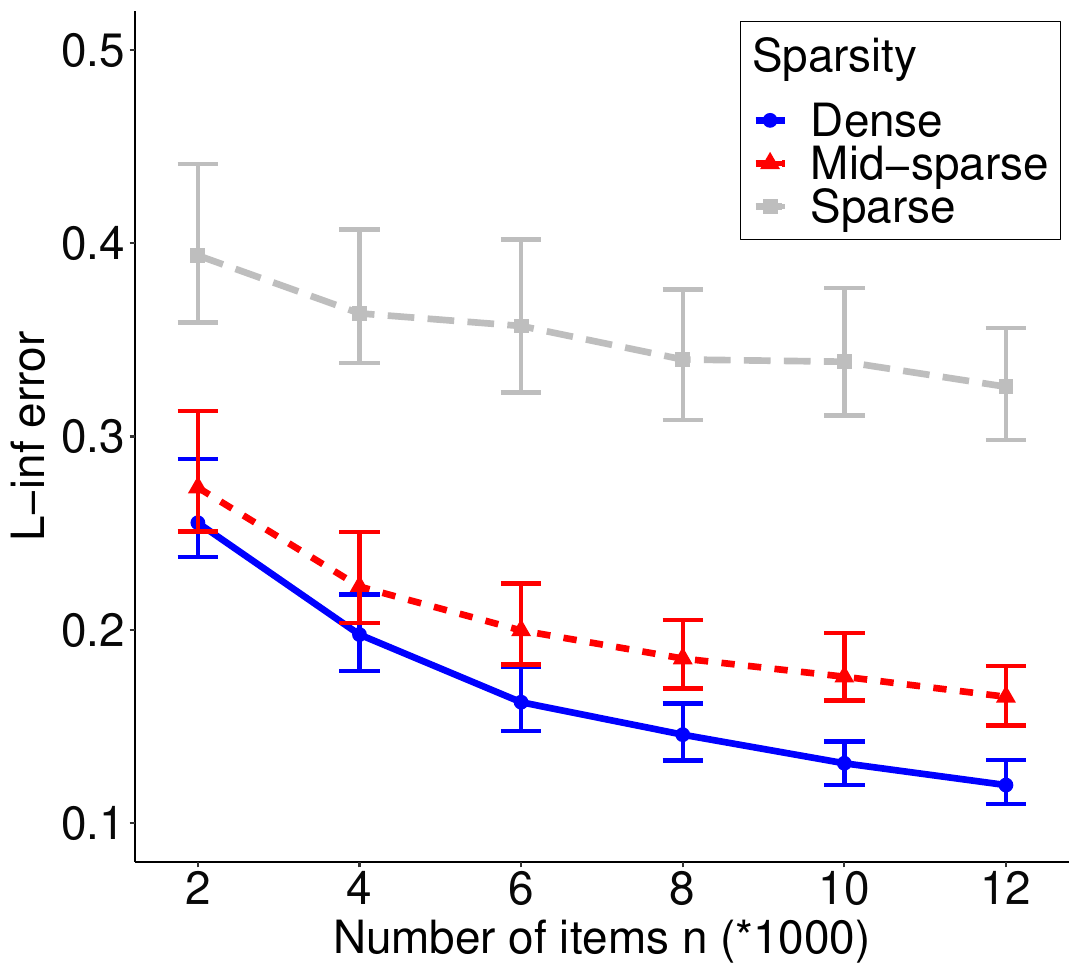}\hspace{1.25 cm}
	\includegraphics[width=0.45\textwidth]{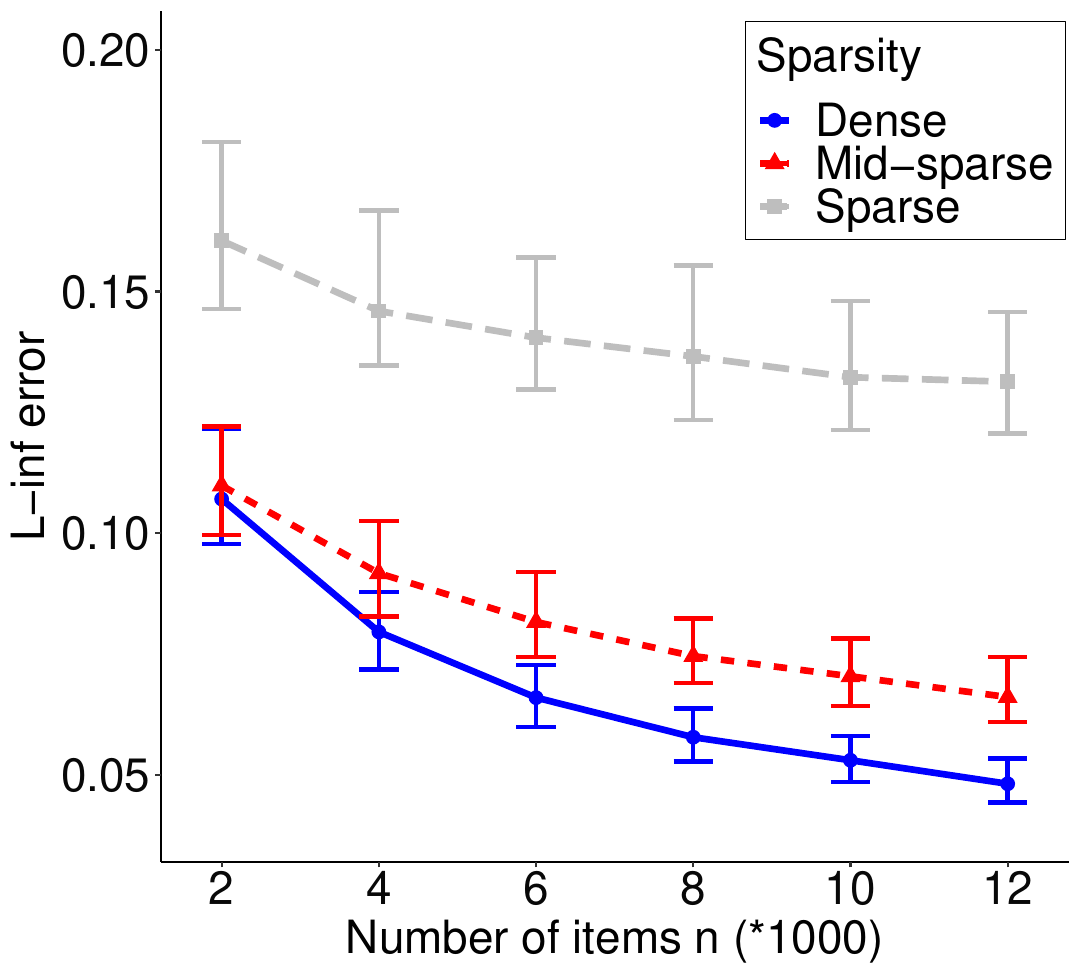}\\
	 \vspace{1 cm}
	\includegraphics[width=0.45\textwidth]{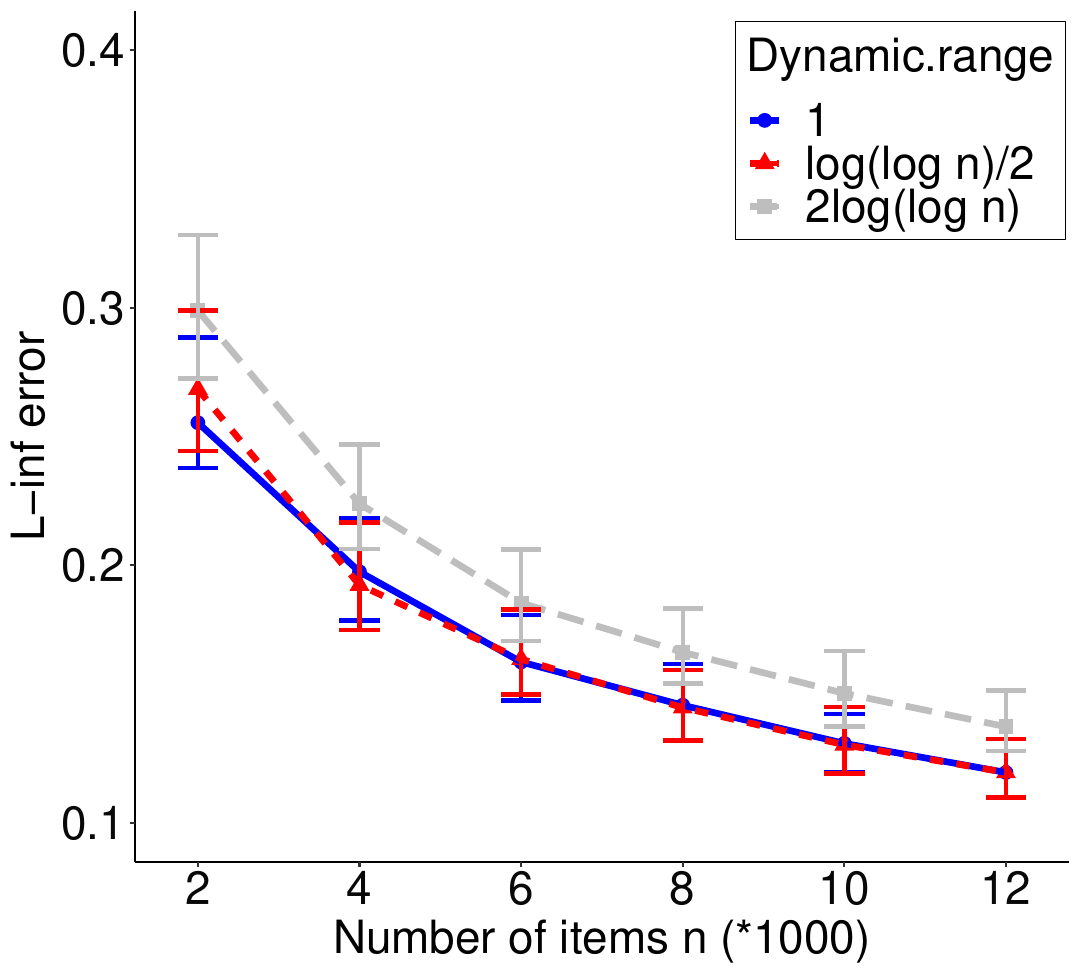}\hspace{1.25 cm}
	\includegraphics[width=0.45\textwidth]{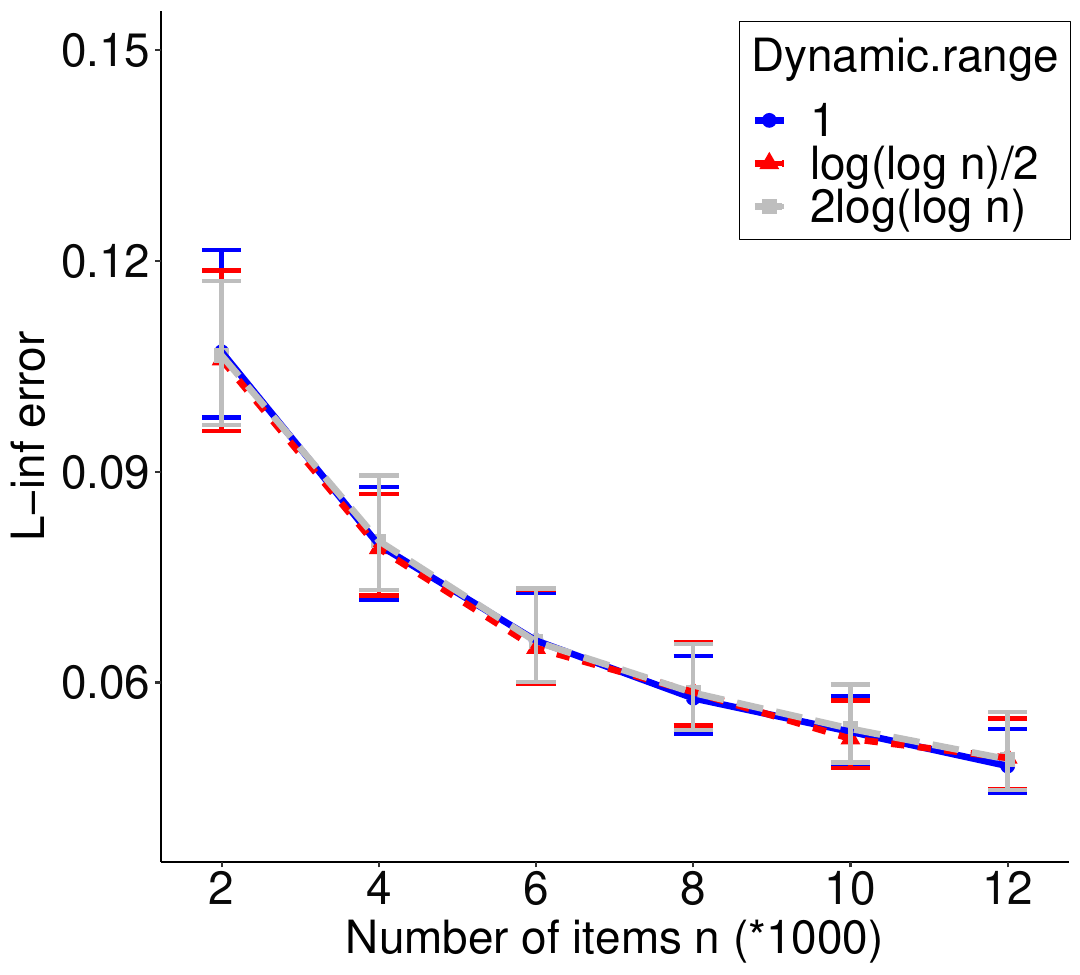}
	\caption{\small Convergence of the MLE in the Davidson model with threshold parameter $\theta=1$ (top left) and the paired cardinal model with variance parameter $\sigma^2=1$ (top right) under three different sparsity regimes: dense $(p_n = 1/2)$, mid-sparse $(p_n = \sqrt{n^{-1}(\log n)^3})$ and sparse $p_n = n^{-1}(\log n)^3$. 
	Convergence of the MLE in the Davidson model with threshold parameter $\theta=1$ (bottom left) and the paired cardinal model with variance parameter $\sigma^2=1$ (bottom right) under three different dynamic ranges, with $p_n = 0.5$. 
	Vertical bars are the quartiles of the $300$ repetitions.}
	\label{davidson}
\end{figure}

In both models and three different sparsity regimes, $\|\widehat{\u}-\u\|_\infty$ decreases to $0$ as $n$ grows to infinity. This numerically verifies the uniform consistency of the MLE as proved in Theorems \ref{main} and \ref{general}. Another observation, which is not unexpected, is that the convergence rate of the MLE closely depends on the density parameter $p_n$: the larger the $p_n$, the faster the convergence. Particularly, when $p_n$ is chosen at the critical level obtained in our analysis, $\|\widehat{\u}-\u\|_\infty$ decays rather slowly compared to the denser regimes. Such drawback seems mitigated by increasing the size of the network, suggesting that networks with extremely large sizes would be more tolerant for a low comparison rate. This demonstrates the potential applicability of our results in studying large complex networks (such as social networks) using under-observed comparison data. 

We also investigate how the convergence of the MLE depends on the varying dynamic range $M_n$. To do so, we fix $ p_n = 0.5$ and take $ M_n $ as $ 1, \log \log n/2 $ and $ 2 \log \log n $, respectively. 
The values of $ M_n $ under different $ n $ can be found in Table \ref{pn_mn}.
According to our results, small $M_n$'s are better for uniform consistency of the MLE in the Davidson model, which is numerically verified in Figure \ref{davidson}.  
As a contrast, the paired cardinal model seems not sensitive to the changing magnitude of $M_n$.  This may be because the convergence rate $ \Delta_n $ in the paired cardinal model is independent of $M_n$; see Corollary \ref{Normal distribution model}. 

\subsection{ATP data analysis}\label{han}
In this section, we model a real pairwise comparison network using three different parametrizations. 
The dataset under our consideration is the ATP dataset \footnote{The dataset can be found at \href{www.tennis-data.co.uk}{www.tennis-data.co.uk}.} from 2000 to 2018. 
The ATP match contains four Grand Slams, the ATP World Tour Masters 1000, the ATP World Tour 500 series, and several tennis series of the year.
For convenience, we focus on the Best of 3 (BO3) matches, which contain two or three sets and have four possible outcomes $ A = \{-2, -1, 1, 2\} $.
For example, the outcome $2:1$ between $i$ and $j$ corresponds to  $ X_{ij} = 2-1 = 1. $ 
As a result, we remove the competitions from Grand slams. 
To let Condition \ref{Condition: existence} hold, we tease out the players who never win or lose the games. 
After cleaning, the dataset includes nearly 26,000 competitions and 954 players.

For model parametrization, we consider three relevant models: the general BT model, the cumulative link model (CLM4), and the adjacent categories logit model (ACLM4) with four outcomes  \citep{agresti1992analysis}.
Specifically, in the general BT model, if there is a comparison between $ i $ and $ j $, then
\begin{eqnarray}\label{1111}
\P(\text{the outcome is 2:1}) &=& 2\Phi^2(u_i - u_j)(1 - \Phi(u_i - u_j)),\nonumber\\
\P(\text{the outcome is 2:0}) &=& \Phi^2(u_i - u_j)
\end{eqnarray}
where $\Phi(x)$ is the logistic link function. 
It can be verified that \eqref{1111} implies that the results of the three sets are independent, and each of them follows the BT model.
In CLM4, 
\begin{eqnarray*}\label{2222}
\P(\text{the outcome is 2:1}) &=& \frac{(\theta - 1)e^{u_i - u_j}}{(\theta + e^{u_i - u_j})(1 + e^{u_i - u_j})},\nonumber\\
\P(\text{the outcome is 2:0}) &=& \frac{e^{u_i - u_j}}{\theta + e^{u_i - u_j}}, \ \theta > 1.
\end{eqnarray*}
In ACLM4, 
\begin{eqnarray*}\label{3333}
	\P(\text{the outcome is 2:1}) &=& \frac{\theta e^{2(u_i - u_j)/3}}{1 + \theta e^{(u_i - u_j)/3} + \theta e^{2(u_i - u_j)/3} + e^{u_i - u_j}},\nonumber\\
	\P(\text{the outcome is 2:0}) &=& \frac{e^{u_i - u_j}}{1 + \theta e^{(u_i - u_j)/3} + \theta e^{2(u_i - u_j)/3} + e^{u_i - u_j}}, \ \theta > 0.
\end{eqnarray*}
It can be checked that all models are within the framework developed in Section \ref{section:setup} and satisfy the log-concavity condition.

We apply several model selection criteria, including the Akaike information criterion (AIC), Bayesian information criterion (BIC) and Leave-one-out
cross-validation (LOOCV), for model evaluation. 
In particular, we choose the prediction error in LOOCV as the cross entropy (negative log-likelihood). For example, if the validation data is the comparison between $ i $ and $ j $ with outcome $ a^* $, then the prediction error is given as
\begin{equation*}
\text{error} = -\sum_{a \in A} {1}_{\{a = a^*\}} \log f(a; \widehat{u}_i - \widehat{u}_j).
\end{equation*}


The result is presented in Table \ref{ATP}. 
We observe that both ACLM4 and CLM4 yield better performance than the general BT model in terms of AIC, BIC, and LOOCV, for which a possible explanation is that the outcomes of different sets in the same match are not independent.
Without assuming independence, both ACLM4 and CLM4 seem to lead to a better fit for the data. 
On the other hand, LOOCV reflects the overall prediction error. 
If we consider using random guessing as to the benchmark, the LOOCV of random guessing is 1.3863 ($ \log 4 $). 
Consequently, all the models we have tested in this example (the general BT model, ACLM4, and CLM4) achieve better predictions than random guessing.

\begin{table}
	\centering
	\begin{tabular}{|c|c|c|c|}
		\hline
		$n$     & General BT model        & ACLM4       & CLM4      \\ \hline
		AIC  & 67810.5 & 67471.1 & 66951.1 \\ \hline
		BIC  & 75601.6 & 75270.3 & 74750.3 \\ \hline
		LOOCV  & 1.3211 & 1.2420 & 1.2374  \\ \hline
	\end{tabular}
	\caption{\small Model selection results of the general BT model, ACLM4, CLM4 applied to the ATP dataset using AIC, BIC and LOOCV.}
	\label{ATP}
\end{table} 

\color{black}

\section{Discussion}\label{section:discussion}


In this paper, we introduced a general framework for statistical network analysis using pairwise comparison data. Our framework enjoys abundant parametrization flexibility for practical purposes. 
Assuming the link function is \emph{valid} and \emph{strictly log-concave} with respect to the parametrization variable, we provided a sufficient condition for the MLE to be uniformly consistent, which can be summarized as a graph topological property.
In particular, for almost homogeneous random graph models (e.g. Erd\H{o}s-R\'enyi graphs), the condition is satisfied when $p_n=\Omega(n^{-1}(\log n)^{3+\epsilon})$ for any $\epsilon>0$, which almost matches the best possible lower bound $n^{-1}\log n$. 
{We think the potential gap between our result and the lower bound is an artifact of our proof, 
which arises owing to a cumulative effect of the large deviation bound in the maximal inequality and the number of times that the maximal inequality is applied for chaining.  
A more elegant design for the proof is needed to avoid applying the maximal inequalities multiple times if one wants to pursue an optimal result. We leave it as one possible direction for future work.}



Although the framework considered in this paper is rather inclusive, there are also some other possible directions one can go to make it better.  For example, one may incorporate a global parameter into the framework to model environmental factors. For instance, the home-field advantage model \citep{MR3887567} contains a global parameter measuring the strength of home-field advantage which is not contained in the latent score vector $\u$. The distribution of the outcome will be different depending on which subject is at home.


Secondly, the assumption on pairwise comparison data can be generalized to multiple comparison data. For example, the Plackett-Luce model \citep{plackett1975analysis,MR0108411} is the multiple-comparison version of the BT model. Compared to pairwise comparison models, multiple comparison models involve data measuring the interaction between more than two items in a single observation, resulting in the comparison graph being a hypergraph. This may cause difficulty in obtaining the asymptotic properties of the MLE. Particularly, the entry-wise error of the MLE in the multiple comparison models is currently elusive to us.



Finally, it is worth pointing out that although our framework allows great flexibility in terms of selecting the link function, it is generally unknown which choice fits the true model best. This poses the natural question for model selection among different valid parametrizations. 
In Section \ref{han}, we conducted an empirical study towards this direction using existing model selection techniques. However, it is not clear to us which method has the best capability of deciphering the most relevant model (both in theory and in practice), which we leave as a direction for future investigation.

\section*{Acknowledgments}
The authors are very grateful to the Editor Prof. McKeague, the Associate Editor, and two anonymous referees for their very helpful comments which significantly improved the presentation of the paper. 
	The authors also thank Prof. Tom Alberts for going through an early version of the draft, Prof. Fan R. K. Chung for explaining a proposed graph condition in the manuscript, and Prof. Zhigang Bao for helpful discussion. 

\section*{Appendices}

\begin{appendices}

In this section, we summarize the notation that will be used throughout the appendices.
\begin{itemize}
	\item $[n]: =\{1,\ldots,n\}$, where $ n $ is the number of subjects.
	\item $ \u = (u_1,\ldots,u_n)^T$: true score vector of the model, where $ u_i$ is the latent score of subject $ i $ for $ i \in [n]$. 
	\item $ M_n $: the dynamic range of $ \u $, which is defined as $ \max_{i,j\in [n]}|u_i - u_j|. $
	\item $ n_{ij} $: the number of the comparisons between subjects $ i $ and $ j $. Specifically, $ n_i = \sum_{j} n_{ij}.$
	\item $ A $: the set of all possible outcomes.
	\item $ X_{ij}^{(t)} $: the outcome of comparison between $ i $ and $ j $ in their $ t$-th comparison, where $ t \in [n_{ij}] $ and $X_{ij} \in A.  $
	
	\item $ C_n^{(i)} $: some constants which depend only on $ f(x;y) $ and $ M_n $,  for $ i\in [4]. $
	\item $ \theta $: the threshold parameter in the Rao-Kupper model and the Davidson model.
	
	\item $ G = (V, E)$: general graph, where $ V $ stands for vertices and $ E $ stands for edges. Specifically, in the main article, the comparison graph is denoted by $G_n= (V_n, E_n) $ generated by $ \{n_{ij}\}_{i \neq j} $, where the subscript emphasizes the dependence on $n$. 
	
		
	
	\item For any $U\subset V$, its boundary edges are defined as $\partial U = \{(i,j)\in E: i\in U, j\in U^\complement\}$; its volume is defined as $\text{vol}(U)=\sum_{u\in U}\deg(u)$. 
 
	\item  For $U_1, U_2\subset V$, $\EE(U_1, U_2)$  denotes the set of cross edges between $U_1$ and $U_2$, that is, $\EE(U_1, U_2) = \{(i, j)\in E: i\in U_1, j\in U_2\}$.

	\item Ratio cut $h_{G}(U):={|\partial U|}/{\min\{|U|, |U^\complement|\}}$.

	\item Modified Cheeger constant $ 	h_G  := \min_{U\subset V} h_G(U).$

		\item For any $ i,j \in V$, dist$(i, j)$ is the shortest-path distance between $  i $ and $  j $.
		\item  For $i\in V$, $\deg(i)$ denotes the degree of $i$.

\item For any $U\subset V$, its $r$-graph neighborhood is defined as $N(U, r)=\{j\in V: \exists i\in U\ \text{such that}\ \text{dist}(j, i)\leq r\}$; its boundary points are defined as $\delta U=N(U, 1)\setminus U$;  its volume is defined as $\text{vol}(U)=\sum_{u\in U}\deg(u) =\sum_{u\in U}n_u$.

\end{itemize}

\section{Proof of Existence and Uniqueness of the MLE}

In this section, we will demonstrate that both the existence and uniqueness of the MLE hold with probability $ 1 $, which is the combination of Theorem 1 and Lemma 1 in the manuscript.
\subsection{Proof of Lemma \ref{exist1}}\label{eeee}
\begin{proof}[Proof of Lemma \ref{exist1}]	
	Recall the log-likelihood function
	\begin{align*}
		&l(\v) = \sum_{1\leq i < j\leq n} \sum_{t\in [n_{ij}]} \log f(X_{ij}^{(t)};v_i - v_j),
	\end{align*}
	where $\v\in\mathbb{V}=\{\v \in \mathbb{R}^n: v_1 = 0\}$. 	Uniqueness would follow if $l(\v)$ is strictly concave on $\mathbb{V}$. Indeed, this can be verified via direct computation: For any $p\in [0,1]$ and $\v, \w\in\mathbb{V}$, 
	\begin{align*}
		&\ pl(\v)+(1-p)l(\w)\\
		=&\  \sum_{1\leq i < j\leq n} \sum_{t\in [n_{ij}]} \left(p\log f(X_{ij}^{(t)}; v_i- v_j)+(1-p)\log f(X_{ij}^{(t)}; w_i- w_j)\right)\\
		\leq&\  \sum_{1\leq i < j\leq n} \sum_{t\in [n_{ij}]} \log f(X_{ij}^{(t)}; pv_i - pv_j+(1-p)w_i-(1-p)w_j)\\
		=&\ l(p\v+(1-p)\w),
	\end{align*}
	where the inequality follows from the assumption that $f(x; y)$ is log-concave with respect to $y$. This proves that $l(\v)$ is concave on $\mathbb{V}$. The concavity is in fact strict since $f(x; y)$ is strictly log-concave with respect to $y$ and the comparison graph is connected. 
	
	For the existence, we note that under Assumption 4, there exists some $C>1$ depending only on $X^{(t)}_{ij}$ such that 
	\begin{align}
		\max_{1\leq i<j\leq n, t\in [n_{ij}]}\sup_{y\in\R}f(X^{(t)}_{ij}, y)<C. \label{C-bound}
	\end{align}
	Particularly, 
	\begin{align*}
		\sup_{\v\in\mathbb{V}}l(\v)\leq \sup_{\v\in\mathbb{V}}\sum_{1\leq i < j\leq n} \sum_{t\in [n_{ij}]} \log C<\infty. 
	\end{align*}
	We claim that the supremum can be achieved in $\mathbb{V}$. Let $\{\v_k\}_{k\in\N}$ be a maximizing sequence of $l(\v)$ in $\mathbb{V}$ such that $l(\v_k)\leq l(\v_{k+1})$, i.e., one may choose $ \v_k \in \argmax_{\v \in \mathbb{V}_k} l(\v)$  with $ \mathbb{V}_k = \{\v \in \mathbb{V}:\left\lVert \v \right\rVert_\infty\leq k\} $.  
	Our claim would be true if $\{\v_k\}_{k\in\N}$ is contained in some compact subset of $\mathbb{V}$, which can be equivalently stated as,  for $i\in [n]$, $\{(\v_k)_i\}_{k\in\N}$, the restriction of $\{\v_k\}_{k\in\N}$ to the $i$-th component, is uniformly bounded. To see this, consider the following sets of components that potentially diverge to $\infty$:
	\begin{align*}
		&S_+:=\{i\in [n]:  \limsup_{k}(\v_k)_i=+\infty\},&S_-:=\{i\in [n]:  \liminf_{k}(\v_k)_i=-\infty\}.
	\end{align*}
	It suffices to show that $S_+=S_-=\varnothing$. Indeed, if $S_+\neq\varnothing$, consider the partition of $[n]$ as $S_+\cup S^\complement_+$. It is clear that $1\in S^\complement_+\neq\varnothing$. According to Condition 1, for $i\in S_+$ there exists $j\in S^\complement_+$ and $t\in [n_{ij}]$ such that $X^{(t)}_{ij}<0$. By Assumption 3, which states that $\log f(x; y)\rightarrow -\infty$ as $y\rightarrow +\infty$ for $x<0$,  we can find $R>0$ such that,
	\begin{equation*}
		\log f(X^{(t)}_{ij}, y)< l(\v_1)-\sum_{1\leq i < j\leq n} \sum_{t\in [n_{ij}]} \log C, \ y\geq R. 
	\end{equation*} 
	Equivalently, for $\v\in\mathbb{V}$ with $v_i-v_j > R$,
	\begin{align*}
		l(\v)\leq \log f(X^{(t)}_{ij}, R)+\sum_{1\leq i < j\leq n} \sum_{t\in [n_{ij}]} \log C<l(\v_1)\leq l(\v_k), 
	\end{align*}
	where the first inequality follows from \eqref{C-bound}.
	This implies that $(\v_k)_j\geq (\v_k)_i - R$ for $k\in \N$. Since $i\in S_+$, we must also have $j\in S_+$ by definition, which contradicts $j\in S_+^\complement$. Therefore, $S_+=\varnothing$. By a similar argument, $S_-=\varnothing$. $S_+=S_-=\varnothing$ together implies that 
	\begin{align*}
		\max_{i\in [k]}\left\{\left|\limsup_{k\to\infty}(\v_k)_i\right|, \left|\liminf_{k\to\infty}(\v_k)_i\right|\right\}<\infty. 
	\end{align*} 
	Hence, $\left\{\v_k\right\}_{k\in\N}$ is bounded thus admits a convergent subsequence whose limit is in $\mathbb{V}$, proving the existence of the MLE.   
\end{proof}

\subsection{Proof of Theorem \ref{exist}}\label{lkg}
\begin{proof}[Proof of Theorem \ref{exist}]	
Without loss of generality, we assume $p_{ij,n} = p_n$ for all $i, j\in [n]$ and $ T = 1 $, since Condition \ref{Condition: existence} is retained when the comparison rates are increased.
Let $H_n$ denote the event that Condition \ref{Condition: existence} holds. We will show $ \P(H_n^\complement) \leq n^{-2}$ for sufficiently large $n$ if 
\begin{equation}\label{condition in Theorem 1}
	\frac{\log n}{np_n|\log C_n^{(1)}|}\to 0 \ \ \ \text{as~} n \to \infty.
\end{equation}
Since $\sum_{n} n^{-2}<\infty$, our desired result follows by applying the Borel-Cantelli lemma, which implies $\P(\limsup_{n}H_n^\complement	) = 0$.

Intuitively, 	\eqref{condition in Theorem 1} suggests that $C_n^{(1)}$ is away from $1$. Note that $C_n^{(1)}$ roughly measures the chance of the strongest subject in the group beating the weakest one in a single comparison. $C_n^{(1)}$ being away from $1$ simply means that there is a certain probability that the weakest subject in the group will win over the others. If enough comparisons take place,  $H_n$ will hold with high probability. 

To make the idea precise, we recall a few notation. 
For any $S\subset [n]$, $|\partial S|$ is the number of cross edges between $S$ and $S^\complement$, that is, $ |\partial S| = \sum_{i \in S, j\in S^\complement } n_{ij}.$
Conditional on $\{n_{ij}\}_{1\leq i<j\leq n}$, applying a union bound over the partitions of $[n]$ yields
\begin{align*}
	\P\left(H_n^\complement\right)&\leq \sum_{r \in [n-1]}\sum_{S\subset [n]: |S| = r}\P\left(\text{Condition \ref{Condition: existence} fails for the partition}\  (S, S^\complement)\right)\\
	&{\leq}  \sum_{r \in [n-1]}\sum_{S\subset [n]: |S| = r} (C_n^{(1)})^{|\partial S|}.
\end{align*}
where the last inequality follows from that the probability that one subject beats the other is at most $C_n^{(1)}$.
For fixed $S$ with $|S|=r$,  $|\partial S|$ is a sum of $r(n-r)$ i.i.d. Bernoulli random variables with parameter $p_n$. Applying the Chernoff bound \citep{MR0057518} together with a union bound over the partitions of $[n]$ yields that,  for sufficiently large $n$, 
\begin{align}
	&\P\left(|\partial S|\geq \frac{1}{2}|S|(n-|S|)p_n, \forall S\subset [n]\right)\label{deg}\\
	\geq\ & 1-\sum_{r\in [n-1]}\sum_{S\subset [n]: |S| = r} \exp\left\{-\frac{1}{8}r(n-r)p_n\right\}\nonumber\\
	\geq\ & 1-\sum_{r\in [n-1]}{n\choose r}\exp\left\{-\frac{1}{8}r(n-r)p_n\right\}\nonumber\\
	\geq\ &1-\sum_{r\in [n-1]}n^{\min\{n-r, r\}}\exp\left\{-\frac{1}{16}n\min\{n-r, r\}p_n\right\}\nonumber\\
	\geq\ &1-\sum_{r\in [n-1]}\exp\left\{-\min\{n-r, r\}\log n\left(\frac{np_n}{16\log n}-1\right)\right\}\nonumber\\
	\geq\ &1-\exp\left\{-\log n\left(\frac{np_n}{16\log n}-2\right)\right\}\nonumber\\
	\geq\ & 1- n^{-3}\nonumber, 
\end{align}  		
where the last inequality holds if $\log n/np_n \to 0$ as $ n \to \infty $.
Denote the event in \eqref{deg} by $F_n$. Our desired result is obtained immediately from straightforward computation: For sufficiently large $n$, by the total probability formula,  
\begin{align*}
	\P\left(H_n^\complement\right)&=\P\left(H_n^\complement|F_n\right) \P\left(F_n\right) + \P\left(H_n^\complement|F_n^\complement\right)\P\left(F_n^\complement\right)\\
	&\leq \P\left(H_n^\complement|F_n\right) + \P\left(F_n^\complement\right)\\
	&\leq \sum_{r\in [n-1]}{n\choose r}(C_n^{(1)})^{\frac{1}{2}r(n-r)p_n}+n^{-3}\\
	&\leq 2\sum_{r\in [\lfloor n/2\rfloor]}{n\choose r}(C_n^{(1)})^{\frac{1}{4}rnp_n}+n^{-3}\\
	&\leq 2\left(\left(1+(C_n^{(1)})^{\frac{1}{4}np_n}\right)^n-1\right)+n^{-3} \stackrel{\eqref{condition in Theorem 1}}{\leq}\ 2n^{-3}\leq n^{-2}. 
\end{align*}	
\end{proof}

\section{Proof of Theorem \ref{main: general}}\label{section:proof}
Next, we prove the uniform consistency of the MLE in the general comparison model under a sequence of  rapidly expanding (RE) comparison graph. 
\begin{proof}[Proof of Theorem \ref{main: general}]

Since $u_1= \widehat{u}_1 = 0$,  we have $\max_{i\in [n]}(\widehat{u}_i-u_i)\geq 0$, $\min_{i\in [n]}(\widehat{u}_i-u_i)\leq 0$, and 
\begin{align*}
	\|\widehat{\u}-\u\|_\infty=\max\{\max_{i\in [n]}(\widehat{u}_i-u_i), -\min_{i\in [n]}(\widehat{u}_i-u_i)\}\leq \max_{i\in [n]}(\widehat{u}_i-u_i)-\min_{i\in [n]}(\widehat{u}_i-u_i).
\end{align*} 
Thus,  to prove Theorem 4, it suffices to show that for sufficiently large $n$, with probability at least $1-n^{-2}$, 
\begin{align}
	\max_{i\in [n]}(\widehat{u}_i-u_i)-\min_{i\in [n]}(\widehat{u}_i-u_i) \lesssim \Delta_n^{RE}. \label{6:goal}
\end{align} 
Define
\begin{align*}
	&\alpha\in \arg\max_{i\in [n]}(\widehat{u}_i  - u_i) & \beta \in \arg\min_{i\in [n]}(\widehat{u}_i  - u_i).
\end{align*}
We wish to show that $(\hat{u}_\a-u_\a)-(\hat{u}_\b-u_\b)$ is bounded by $\Delta_n^{RE}$ as $n$ tends to infinity with high probability. However, direct evaluation of $(\widehat{u}_\a-u_\a)-(\widehat{u}_\b-u_\b)$ is impossible unless there is a comparison between $\a$ and $\b$, which will become less likely in sparse graphs. Such phenomenon poses a great challenge in analyzing uniform consistency properties of the MLE in the context of sparse comparison networks.
We will address this issue by applying a novel chaining argument. 

Let $ c $ be an absolute constant, $K_{n,1}$  and $K_{n,2}$ be integers depending on $ n $,  $ \{\Delta d_k\}^{K_{n,1}}_{k=1} $ and $ \{\Delta s_k\}^{K_{n,2}}_{k=1} $ be increasing sequences. Consider the following upward-nested sequences of neighborhood at $\a$ and $\b$ respectively based on the estimation error:
\begin{align*}
	&B_k = \left\{j: (\widehat{u}_\alpha - u_\alpha) - (\widehat{u}_j - u_j) \leq \sum_{\ell=0}^{k-1}\Delta d_\ell  \right\} & \Delta d_k = \omega_n\sqrt{\frac{4c\log n }{h_{G_n}(B_k)}}\\
	& C_k = \left\{j: (\widehat{u}_\beta - u_\beta) - (\widehat{u}_j - u_j) \geq -\sum_{\ell=0}^{k-1}\Delta s_\ell  \right\}& \Delta s_k = \omega_n\sqrt{\frac{4c\log n }{h_{G_n}(C_k)}},
\end{align*}
where $ \Delta d_0 = \Delta s_0 = 0 $, $K_{n,1}$ and $K_{n,2}$ as the stopping times
\begin{align*}
	&K_{n,1} = \left\{k: |B_k|>\frac{n}{2}\right\}& K_{n,2} = \left\{k: |C_k|>\frac{n}{2}\right\}.
\end{align*}
We make the following claim:

\begin{claim}\label{myclaim}
	There exists an absolute constant $c$ such that for all sufficiently large $n$, $\{B_k\}^{K_{n,1}}_{k=1}$ and $\{C_k\}^{K_{n,2}}_{k=1}$ constructed above are admissible sequences of $G_n$. 
\end{claim}

We now finish the proof assuming Claim \ref{myclaim} is true. 
Since $|B_{K_{n,1}} |>n/2$ and $|C_{K_{n,2}}|>n/2$, $B_{K_{n,1}}\cap C_{K_{n,2}} \neq\varnothing$. 
Thus, there exists an element $\nu\in B_{K_{n,1}}\cap C_{K_{n,2}}$ to chain $\alpha$ with $\beta$:
\begin{align*}
	(\widehat{u}_\a-u_\a)-(\widehat{u}_\b-u_\b) &= (\widehat{u}_\a-u_\a)-(\widehat{u}_\nu-u_\nu)+(\widehat{u}_\nu-u_\nu)-(\widehat{u}_\b-u_\b)\\
	&\leq \sum_{k=1}^{K_{n,1}-1}\omega_n\sqrt{\frac{4c\log n }{h_{G_n}(B_k)}} + \sum_{k=1}^{K_{n,2}-1}\omega_n\sqrt{\frac{4c\log n }{h_{G_n}(C_k)}}\\
	&\leq 4\sqrt{c}\omega_n\max_{\{A_k\}_{k=1}^K\in \mathscr A(G_n)}\sum_{k=1}^{K-1} \sqrt{\frac{\log n}{h_{G_n}(A_k)}} =  4\sqrt{c}\omega_n\Gamma_{G_n}^{RE},
\end{align*}
where the last inequality follows from the assumption that $\{G_n\}_{n\in\N}$ is RE. 
The proof is complete. 
We end this section by providing the proof of Claim \ref{myclaim}.

\textit{Proof of Claim \ref{myclaim}.}	
We only prove for $\{B_k\}^{K_{n,1}}_{k=1}$ as the case of $\{C_k\}^{K_{n,2}}_{k=1}$ can be shown similarly. Firstly, note for $k<{K_{n,1}}$, 
\begin{align}
	\Delta d_k = \omega_n\sqrt{\frac{4c\log n }{h_{G_n}(B_k)}} \stackrel{|B_k|\leq \frac{n}{2}}{=} \omega_n\sqrt{\frac{4c|B_k|\log n }{|\partial B_k|}}.\label{bks}
\end{align}
Since the comparison graph sequence is RE with rate $\omega_n$, it is AE with the same rate.
Thus, for all sufficiently large $n$, 
\begin{align}
	\max_{k\in[K_{n,1} -1]}\Delta d_k\leq 1.\label{deltad}
\end{align}  
Secondly, observe the symmetry under Assumption 2: $\log f(x;y) = \log f(-x;-y)$. Taking derivative on both sides with respect to $y$ yields
\begin{align}
	g(x;y) + g(-x; -y) = 0.\label{keyobs}
\end{align}
Recall the log-likelihood function can be written as
\begin{equation*}
	l(\v) = \frac{1}{2}\sum_{i\in [n]}\sum_{j\in\delta\{i\}} \sum_{t\in [n_{ij}]} \log f(X_{ij}^{(t)};v_i - v_j).
\end{equation*}
Since $\widehat{\u}$ is the maximizer of $l(\v)$ under constraint $v_1=0$, the first-order optimality condition implies that
\begin{align}
	&\partial_i l(\widehat{\u})=\sum_{j\in\delta\{i\}} \sum_{t\in [n_{ij}]} g(X_{ij}^{(t)}; \widehat{u}_i-\widehat{u}_j)=0& i\in [n]\setminus\{1\}, \label{l1:1}
\end{align}
where $\partial_i l$ denotes the derivative of $l$ with respect to the $i$-th component. 
For fixed $i\in [n]$, we consider
\begin{equation*}
	\partial_i l(\u)=\sum_{j\in\delta\{i\}} \sum_{t\in [n_{ij}]}g(X_{ij}^{(t)}; u_i-u_j).\label{l1:2}
\end{equation*}
It is easy to verify $\E[\partial_i l(\u)]=0$. 
For every fixed $S\subset [n]$ satisfying $ |S| \leq  n/2 $, since $\{g(X_{ij}; u_i-u_j)-\E[g(X_{ij}; u_i-u_j)]\}_{1\leq i<j\leq n}$ is a sequence of independent sub-Gaussian random variables with sub-Gaussian norms bounded by $C_n^{(4)}$, applying Hoeffding's inequality yields that, with probability at least $1-n^{-4|S|}$,
\begin{align}
	\sum_{i\in S}(\partial_i l(\u)-\partial_i l(\widehat{\u})) &\stackrel{\eqref{l1:1}}{=} \sum_{i\in S}(\partial_i l(\u)-\E[\partial_i l(\u)])\label{myeq1}\\
	& = \sum_{i\in S}\sum_{j\in\delta\{i\}} \sum_{t\in [n_{ij}]}\left(g(X_{ij}^{(t)}; u_i-u_j)-\E[g(X_{ij}^{(t)}; u_i-u_j)]\right)\nonumber\\
	& \stackrel{\eqref{keyobs}}{=}\sum_{i\in S}\sum_{j\in\delta\{i\}\cap S^\complement} \sum_{t\in [n_{ij}]}\left(g(X_{ij}^{(t)}; u_i-u_j)-\E[g(X_{ij}^{(t)}; u_i-u_j)]\right)\nonumber\\
	&\leq C_n^{(4)}\sqrt{c|S||\partial S|\log n}\nonumber, 
\end{align} 
where $c$ is an absolute constant.  
A similar estimate of \eqref{myeq1} can be obtained for $B_k$ (which is random) using a union bound:
\begin{align}
	&\P\left(\sum_{i\in B_k}(\partial_i l(\u)-\partial_i l(\widehat{\u}))<C_n^{(4)}\sqrt{c|B_k||\partial B_k|\log n}\right)\label{myeq3}\\
	=&\ 1-\P\left(\sum_{i\in B_k}(\partial_i l(\u)-\partial_i l(\widehat{\u}))\geq C_n^{(4)}\sqrt{c|B_k||\partial B_k|\log n}\right)\nonumber\\
	\geq&\ 1-\sum_{S\subset [n]: |S| \leq  n/2 }\P\left(\sum_{i\in S}(\partial_i l(\u)-\partial_i l(\widehat{\u}))\geq C_n^{(4)}\sqrt{c|S||\partial S|\log n}\right)\nonumber\\
	\stackrel{\eqref{myeq1}}{\geq}&\ 1-\sum_{s\in [\lfloor n/2\rfloor]}\sum_{S\subset [n]: |S| = s}n^{-4s}\nonumber\\
	=&\ 1-\sum_{s\in [\lfloor n/2\rfloor]}{n\choose s}n^{-4s}\nonumber\\
	\geq&\ 1-n^{-2}.\nonumber
\end{align}
For $k<K_{n,1}$, 
\begin{align}
	\sum_{i\in B_k}(\partial_i l(\u)-\partial_i l(\widehat{\u}))& \stackrel{\eqref{keyobs}}{=} \sum_{i\in B_k}\sum_{j\in\delta\{i\}\cap B_k^\complement} \sum_{t\in [n_{ij}]} \underbrace{\left(g(X_{ij}^{(t)}, u_i-u_j)-g(X_{ij}^{(t)}, \widehat{u}_i-\widehat{u}_j)\right)}_{\text{$\geq 0$ since $g$ is nonincreasing in $y$}}\nonumber\\
	&\geq \sum_{i\in B_k}\sum_{j\in\delta\{i\}\cap B_{k+1}^\complement} \sum_{t\in [n_{ij}]}\left(g(X_{ij}^{(t)}, u_i-u_j)-g(X_{ij}^{(t)}, \widehat{u}_i-\widehat{u}_j)\right)\label{myeq2}.
\end{align}
Note for $i\in B_k$ and $j\in B_{k+1}^\complement$,
\begin{align*}
	\widehat{u}_i-\widehat{u}_j - (u_i-u_j) = (\widehat{u}_\alpha-u_\alpha-(\widehat{u}_j-u_j))-(\widehat{u}_\alpha-u_\alpha-(\widehat{u}_i-u_i))\geq \Delta d_k.
\end{align*}
We can further bound \eqref{myeq2} using the mean-value theorem:
\begin{align*}
	\sum_{i\in B_k}(\partial_i l(\u)-\partial_i l(\widehat{\u}))&\geq \sum_{i\in B_k}\sum_{j\in\delta\{i\}\cap B_{k+1}^\complement} \sum_{t\in [n_{ij}]} \left(g(X_{ij}^{(t)}, u_i-u_j)-g(X_{ij}^{(t)}, u_i-u_j + \Delta d_k)\right)\\
	& = \sum_{i\in B_k}\sum_{j\in\delta\{i\}\cap B_{k+1}^\complement} \sum_{t\in [n_{ij}]}|g_2(\zeta_{ij})|\Delta d_k\ \ \ \ \ \ \ \ \zeta_{ij}\in [u_i-u_j, u_i-u_j+\Delta d_k]\\
	&\stackrel{\eqref{deltad}}{\geq} \sum_{i\in B_k}\sum_{j\in\delta\{i\}\cap B_{k+1}^\complement} \sum_{t\in [n_{ij}]}C_n^{(3)}\Delta d_k.
\end{align*}
This combined with \eqref{myeq3} implies, for sufficiently large $n$, with probability at least $1-n^{-2}$,
\begin{align}
	C^{(3)}_n|\EE(B_k, B_{k+1}^\complement)|\Delta d_k&\leq \sum_{i\in B_k}(\partial_i l(\u)-\partial_i l(\widehat{\u}))\leq C_n^{(4)}\sqrt{c|B_k||\partial B_k|\log n}.\label{katy}
\end{align}
Under our choice of $\Delta d_k$ in \eqref{bks}, \eqref{katy} implies
\begin{align*}
	|\EE(B_k, B_{k+1}\setminus B_k)|\geq |\EE(B_k, B_{k+1}^\complement)|\Longrightarrow\frac{|\EE(B_k, B_{k+1}\setminus B_k)|}{|\partial B_k|}\geq\frac{1}{2},
\end{align*}
proving that $\{B_k\}_{k\in [K_{n,1}]}$ is an admissible sequence. 
\end{proof}

\section{Proof of Proposition \ref{myprop} and Theorems \ref{main}-\ref{general}}
\subsection{Proof of Theorems \ref{main}-\ref{general}}\label{proof:ASC}
The proof of Theorems \ref{main}-\ref{general} is straightforward as long as we can show the comparison graph defined in Section \ref{section:setup} of the main manuscript is rapidly expanding with $ \Gamma_{G_n}^{RE} \lesssim  \sqrt{{q_n^2(\log n)^3}/{np_n^3}}$. 
In other words, we need to demonstrate the second statement in Proposition \ref{myprop} when involving constant $ T $, for which we provide rigorous proof below. 

\subsection{Proof of Proposition \ref{myprop}}\label{proof:ASCC}
\begin{proof}[Proof of Proposition \ref{myprop}]
We first demonstrate the second statement is true. The first statement is a special case of the second one. 
Recall that the number of comparisons between $ i $ and $ j $, $ n_{ij}$, satisfies $n_{ij} \sim \text{Bin}(T, p_{ij,n}) $, where $ p_n $ and $ q_n $ are taken as the minimum and maximum comparison rate, respectively, that is, $p_{n} := \min_{i,j\in [n]}p_{ij,n}$ and $q_{n} :=  \max_{i,j\in [n]}p_{ij,n}.$

To prove the RE part, note that a general admissible sequence $\{A_k\}_{k=1}^K$ can be divided into two subsequences: $\{A_k\}_{k=1}^{K_1}$ and $\{A_k\}_{k=K_1+1}^K$ with $K_1 = \max\{k: |A_k|\leq n/2\}$. We find the value of $ K_1 $  and $ K - K_1 $ separately.
Using a similar argument as in \eqref{deg}, it follows that, with probability at least $ 1 - n^{-3} $,
\begin{align}
	\frac{T}{2}|S|(n-|S|)p_n \leq |\partial S|, \text{ for any } S\subset [n]. \label{degS}
\end{align}
Meanwhile,
Hoeffding's inequality combined with a union bound yields that, for sufficiently large $n$, with probability at least $1-n^{-3}$, 
\begin{align}
	\frac{T}{2}np_n\leq\min_{i\in V_n}\deg(i)\leq\max_{i\in V_n}\deg(i)\leq 2Tnq_n. \label{degc}
\end{align}
To prove the RE part, we take any $\{A_k\}_{k=1}^{K_1}\in \mathscr A(G_n)$. Due to the definition of admissible sequences, for $ k \leq K_1, $   
 \begin{equation*}
\EE(A_k, A_{k+1}\setminus A_k) \geq |\partial A_k|/2.
 \end{equation*}
 Combining \eqref{degS} with \eqref{degc}, we have
 \begin{equation*}
2Tnq_n |A_{k+1}\setminus A_k|     \geq \EE(A_k, A_{k+1}\setminus A_k)  \geq |\partial A_k|/2 \geq \frac{T}{4}|A_k|(n-|A_k|)p_n.
 \end{equation*}
Since $|A_k|\leq n/2$,
 \begin{equation*}
	 |A_{k+1}\setminus A_k|     \geq  \frac{p_n}{16q_n}|A_k|\Longrightarrow	\frac{|A_{k+1}|}{|A_k|} \geq 1+\frac{p_n}{16q_n}.
\end{equation*}
Since $|A_1| \geq 1$, 
\begin{align*}
	\left(1+\frac{p_n}{16q_n}\right)^{K_1 - 1}\leq |A_{K_1}|\leq \frac{n}{2}\Longrightarrow K\leq \frac{32q_n\log n}{p_n}.
\end{align*}
On the other hand, if $ k > K_1$, namely, $ |A_k| > n/2, $ we obtain a similar result
 \begin{equation*}
	|A_{k+1}\setminus A_k|     \geq  \frac{p_n}{16q_n}(n-|A_k|)\Longrightarrow\frac{|A_{k+1}^\complement|}{|A_k^\complement|} \leq 1-\frac{p_n}{16q_n}.
\end{equation*}
Since $|A_{K-1}^\complement| \geq 1$ and $ |A_{K_1}| > n/2 $, 
\begin{align*}
|A_{K-1}^\complement| \leq	\left(1-\frac{p_n}{16q_n}\right)^{K - K_1 - 1} |A_{K_1}^\complement| \Longrightarrow K - K_1 \leq \frac{32q_n\log n}{p_n}.
\end{align*}
Meanwhile, \eqref{degS} implies $ h_{G_n}(A_k)\geq Tnp_n/4. $ Thus, 
\begin{align*}
	&\sum_{k=1}^{K-1}\sqrt{\frac{\log n}{h_{G_n}(A_k)}}\lesssim \frac{q_n\log n}{p_n}\sqrt{\frac{\log n}{np_n}}.
\end{align*}
Namely, $ \Gamma_{G_n}^{RE} \lesssim  \sqrt{{q_n^2(\log n)^3}/{np_n^3}}$.

\textit{For the third statement}, we take $ T = 1 $ for convenience. Since we require a lower bound on $ p_n, $ we can demonstrate $ h_{G_n}(A_k)\geq np_n/4 $ via \eqref{deg} and \eqref{degS}. We only need to find an upper bound for $ K $. According to the proof of the second statement, we find that the $ |A_k| $ increases exponentially when  $ |A_k| \leq n/2 $ and $ |A_k^\complement|  $ decreases exponentially when $ |A_k| > n/2 $. 
Since $ |A_k| $ cannot exceed $ n $, $ K \lesssim \log n $ in the homogeneous graphs. 
This idea could be extended to the stochastic block model with the finite number of the blocks, say $ t $ blocks. Specifically, we can separate $ V_n = V^{(1)} \cup \ldots \cup V^{(t)} $ to the vertices set of each block. Similarly, $ A_k = A_k^{(1)} \cup \ldots \cup A_k^{(t)} $. The idea is to show that at least one of $ |A_k^{(1)}|, \ldots, |A_k^{(t)}| $ will behave like the $ |A_k| $ under the homogeneous graph, namely $|A_k^{(j)}|$ increases exponentially when $ |A_k^{(j)}| \leq | V^{(j)}|/2  $ and $| V^{(j)}\setminus A_k^{(j)}|$ decreases  exponentially when $ |A_k^{(j)}| > | V^{(j)}|/2  $ for  $ j \in [t]. $
As a result, we require $ \Omega(\log n) $ steps to make $ A_k^{(j)} $ reach $V^{(j)}$, then $ K \lesssim t\log n. $ Since $ t $ is finite and $ h_{G_n}(A_k)\geq np_n/4 $,  $ \Gamma_{G_n}^{RE} \lesssim \sqrt{{(\log n)^3}/{np_n}}. $

The left part is to show  at least one of $ |A_k^{(1)}|, \ldots, |A_k^{(t)}| $ will behave likes the $ |A_k| $ under the homogeneous graph. For ease of illustration, we consider $ t = 2 $ and  corresponding edge density is
$$ \begin{bmatrix}
	p_{11}       & p_{12}  \\
	p_{12}       & p_{22}  \\
\end{bmatrix}. $$
 The proof for $ t>2 $ is similar but more tedious. 
 Consider an admissible sequence $\{A_k\}_{k\in [K]}$ and fix $k<K$. 
 Suppose that  $ V^{(j)} = n_j $, $ A_k^{(j)} = a_j $ and $ |A_k^{(j+1)}\setminus A_k^{(j)}| = m_j, j = 1, 2.$ 
 To reduce clutter, we omit the index dependence on $k$. 
 It is easy to verify that $ n_1 + n_2 = n. $ 

Without loss of generality, we assume $ 1 \leq a_1 < n_1 $ and $ 1 \leq a_1 < n_2 $. Otherwise, one of $ a_1, a_2 $ is always zero before another reaches its maximum value, say $ a_2 = 0. $ This case degenerates to the Erd\H{o}s-R\'enyi graph since we only focus on the first block. Consequently, $ \Omega{(\log n)} $ steps is needed to make $ a_1 $ increase to $ n_1. $ After that, one more step will make $ a_2 > 0 $ and then at most additional $ \Omega{(\log n)} $ steps are required similarly. Therefore $ K \lesssim \log n. $

Now, we focus on $ 1 \leq a_1 < n_1 $ and $ 1 \leq a_1 < n_2. $ Following the argument in \eqref{deg}, we have
\begin{equation}\label{eq: partial A_k}
|\partial A_k| \geq \frac{1}{2} \left[a_1\{(n_1 - a_1)p_{11} + (n_2-a_2)p_{12}\} + a_2\{(n_1 - a_1)p_{12} + (n_2-a_2)p_{22}\}\right]. 
\end{equation}
Besides that, similar degree concentration yields 
$$ \max_{i\in V^{(1)} }\deg(i)\leq 2(n_1p_{11} + n_2p_{12}),\  \max_{i\in V^{(2)} }\deg(i)\leq 2(n_1p_{12} + n_2p_{22})  .$$
According to the definition of the admissible sequence, 
\begin{equation}\label{eq: admiss}
2m_1(n_1p_{11} + n_2p_{12}) +  2m_2(n_1p_{12} + n_2p_{22}) \geq \vol(A_k^{(j+1)}\setminus A_k^{(j)})\geq \frac{1}{2}|\partial A_k|.
\end{equation}
Combining it with \eqref{eq: partial A_k}, we obtain 
\begin{equation}\label{eq: m}
	 m_1 \geq \frac{1}{64}\min\{a_1, n_1 - a_1\} \text{ or } m_2 \geq \frac{1}{64} \min\{a_2, n_2- a_2\}
\end{equation}
which implies $ K \leq 256\log n. $ Indeed, one can check via a contradiction argument that \eqref{eq: admiss} will violate if \eqref{eq: m} does not hold. 
\end{proof}
\end{appendices}

\spacingset{1.9}
\bibliographystyle{apalike}
\bibliography{sample}

\end{document}